%% file: main_tmlr.tex
\documentclass[10pt]{article}

\usepackage[accepted]{tmlr}

\usepackage{amsmath}
\usepackage{amssymb}
\usepackage{amsthm} 
\usepackage{array}
\usepackage{booktabs}
\usepackage{microtype}
\usepackage{xcolor}
\usepackage{mathrsfs}
\usepackage{bbm}

\theoremstyle{plain}
\newtheorem{theorem}{Theorem}[section]
\newtheorem{proposition}[theorem]{Proposition}

\newtheorem{corollary}[theorem]{Corollary}

\theoremstyle{definition}
\newtheorem{definition}[theorem]{Definition}
\newtheorem{assumption}[theorem]{Assumption}

\theoremstyle{remark}
\newtheorem{remark}[theorem]{Remark}

\newcolumntype{L}[1]{>{\raggedright\arraybackslash}p{#1}}
\usepackage{url}

\title{Structuring Relations Among Learning Paradigms\\via Protocol--Objective--Resource Reductions}

\author{\name Junwei Su\textsuperscript{*} \email junweisu.cs@gmail.com \\
      \addr University of Science and Technology of China
      \AND
      \name Changjie Wang \email changjie\_w@mail.ustc.edu.cn \\
      \addr University of Science and Technology of China
      \AND
      \name Dongyang Chang \email dychang@mail.ustc.edu.cn\\
      \addr University of Science and Technology of China 
      }

\def\month{09}  
\def\year{2026} 
\def\openreview{\url{https://openreview.net/forum?id=8ZyXUFRfb2}} 

\begin{document}

\raggedbottom

\maketitle

\input{paper/abstract}
\input{paper/introduction}
\input{paper/preliminary}
\input{paper/main_results}
\input{paper/related_work}
\input{paper/conclusion}
\input{paper/broader_impact}

\bibliography{reference}
\bibliographystyle{tmlr}

\clearpage
\appendix
\input{paper/appendix_discussion}
\input{paper/appendix_proofs}

\end{document}

%% file: paper/abstract.tex
\begin{abstract}
Modern machine learning increasingly spans supervised, transfer, continual,
meta-learning, and other training regimes, often while reusing the same
hypothesis families, architectures, and optimizers. Yet these regimes can
differ substantially in information access, success criteria, memory
assumptions, adaptation protocols, and sample accounting. As a result, it is
often unclear whether a proposed paradigm is genuinely distinct, a special case
of an existing paradigm, or part of a broader structural relationship among
paradigms. This motivates a structural theory for comparing learning paradigms
while separating representational capacity from protocol, objective, and
resource choices. We develop a protocol--objective--resource (POR) framework
for making such comparisons precise. In this framework, a learning paradigm is
specified by its environment class, observation protocol, admissible learner
class, performance functional, and resource accounting rule. We define POR
reductions through environment embeddings, learner compilers, threshold maps,
and calibrated declared-resource overheads. Our main theorem shows that a
POR reduction implies worst-case complexity domination, in the declared
accounting units, on embedded comparison classes; under labeled-example
accounting this is sample-complexity domination. It transfers upper bounds in
one direction and lower bounds in the other. We further provide reusable templates for
proving reductions, show why nontrivial, calibrated accuracy regimes are
necessary to exclude vacuous comparisons, and prove that strengthening the objective can strictly
increase minimax sample complexity even when the protocol and learner class
remain unchanged. We instantiate the framework on supervised, transfer,
continual, and meta-learning abstractions. We operationalize transcript laws and resource costs and adopt a uniform fixed-budget minimax notion. Under explicit interface-compatibility assumptions, the exact hierarchy gives canonical special-case sanity checks: continual contains transfer, transfer contains supervised, and meta contains supervised under aligned raw-example accounting. We also give a non-exact episode-to-example reduction for episodic meta-learning. The same formalism captures within-paradigm refinements such as replay memory and task identifiers. Together, these results structure relations among learning paradigms and enable transfer of complexity guarantees in the declared accounting units.
 \let\thefootnote\relax\footnotetext{$*$: corresponding author}
\end{abstract}

%% file: paper/introduction.tex
\section{Introduction}
\label{sec:introduction}

Reductions are among the most powerful organizing principles in computer
science. The theory of NP-completeness showed that apparently different
computational problems can be related systematically: an instance of one problem
can be encoded as an instance of another, so that a solver for the latter yields
a solver for the former
\citep{cook1971complexity,karp1972reducibility,garey1979computers,papadimitriou1994computational,arora2009computational}.
This idea did more than classify individual problems as hard. It revealed a
network of relationships among difficult problems, allowing algorithms, proof
techniques, lower bounds, and impossibility results to move across problem
formulations rather than being rebuilt from scratch each time.

This structural perspective can help conserve research effort within the
interfaces covered by a reduction. When two formulations reduce to one another,
formal guarantees may be reused; when a reduction is obstructed, the obstruction
identifies which stated protocol, objective, or resource condition differs.
This does not by itself decide empirical or scientific novelty, but it turns an
informal similarity claim into explicit, auditable obligations.

Modern machine learning increasingly needs an analogous language. The field is
organized around supervised learning, transfer learning, continual learning,
meta-learning, and many variants of these regimes
\citep{caruana1997multitask,thrun1998lifelong,pan2010survey,bendavid2010theory,weiss2016survey,parisi2019continual,delange2022continual,vanderven2022three,baxter2000model,finn2017model,hospedales2022meta,su2023towards,su2024pres,su2023limitation}.
These labels are useful, but they can hide distinct sources of difficulty.
Different paradigms often reuse the same hypothesis families, representations,
architectures, losses, and optimizers
\citep{bengio2013representation,caruana1997multitask,crammer2008learning,maurer2016benefit,pentina2015lifelong,tripuraneni2020task,su2024topology,su2025interplay,su2026improving},
while differing in what data are revealed, when the data are revealed, what
counts as success, how memory and adaptation are constrained, and how samples
are charged.
A fixed predictor architecture, for example, may be trained from scratch,
adapted from source tasks, updated sequentially across tasks, or used inside an
episodic meta-learning procedure.

Without a structural theory, it is difficult to state precisely whether one
learning interface is a special case of another or whether their formal
difference lies in protocol, objective, or resource accounting. This question
is narrower than scientific novelty: it concerns which guarantees follow under
specified interfaces. It also makes method-transfer claims auditable. Saying
that a method transfers from one paradigm to another is not only a statement
about using the same model class; it additionally requires the interaction
protocol, success criterion, and charged resource budget to be matched.

Existing learning-theoretic complexity measures do not resolve this issue.
Tools such as VC dimension, Rademacher complexity, covering numbers, stability,
and norm-based capacity bounds compare hypothesis classes under a fixed
learning setup
\citep{vapnik1971uniform,valiant1984theory,blumer1989learnability,vapnik1998statistical,anthony1999neural,shalev2014understanding,mohri2018foundations,bartlett2002rademacher,bousquet2002stability,su2025non,liu2026full,su2026multi}.
They are not designed to separate the effect of representational capacity from
the effect of changing the learner's information interface, objective, memory
constraint, adaptation protocol, or sample-accounting rule. Consequently, a
statement such as ``continual learning is harder than transfer learning'' is
formally incomplete. The claimed gap might come from the need to control
forgetting, the absence of task identifiers, a stronger all-task objective, a
smaller replay budget, or a different convention for counting source and target
samples
\citep{robins1995catastrophic,french1999catastrophic,kirkpatrick2017overcoming,lopezpaz2017gradient,rebuffi2017icarl,delange2022continual,vanderven2022three,kalan2020minimax,sheng2024mspipe,su2025temporal,su2024bg,shen2026structuring}. This paper asks the following question:

\begin{quote}
{\bf \it What exactly makes one learning paradigm harder than another, and under
what conditions can understanding and methods be systematically transferred
across paradigms?}
\end{quote}

We answer this question by developing a reduction-based framework for comparing
learning paradigms. The analogy is with reductions in complexity theory, but the
objects are statistical learning interfaces rather than decision problems. A
learning paradigm is represented as a protocol--objective--resource (POR)
specification: it consists of an environment class, an observation protocol, an
admissible learner class, a performance functional, and a resource accounting
rule. A POR reduction certifies that any successful learner for one paradigm can
be systematically compiled into a successful learner for another, after
translating environments, success requirements, and declared resource budgets.

Operationally, the observation protocol and accounting rule induce a
budgeted execution law on transcripts: an execution at budget \(M\) may reveal
only a transcript whose charged cost is at most \(M\). A reduction must couple
the solver and compiled executions, map every admissible solver learner into an
admissible reduced learner, preserve the required success event, and certify the
affine cost bound. Worst-case complexity in the declared accounting unit is
uniform: a single learner at a single budget must meet the target on every
environment in the stated comparison class.

The practical role of the framework is to separate representation from the rest
of the learning interface. Shared architectures or optimizers do not by
themselves imply that understanding or methods transfer. For understanding
transfer, a reduction states when upper bounds, lower bounds, and impossibility
results in one paradigm imply corresponding results in another. For method
transfer, the same reduction gives a checklist: the target setting must be able
to reproduce the source method's data transcript, source success must imply
target success after the threshold translation, and the simulation must not hide
extra sample, memory, replay, or adaptation cost. This makes informal claims of
transfer auditable rather than merely intuitive.

The formal compiler is an interface-level simulation. Claims that it is
computationally efficient or directly implementable require an additional
constructivity or computational-overhead condition; the present results certify
complexity transfer in the declared accounting unit only.

\paragraph{Contributions.}
The main contributions of this paper are as follows.

\begin{itemize}
    \item \textbf{A reduction framework for learning paradigms.}
    Section~\ref{sec:preliminaries-framework} formalizes learning paradigms as
    POR objects and defines POR reductions. The definition records four pieces
    of reduction data: environments must be embedded, learners must be compiled,
    success requirements must be translated, and resource budgets must be
    calibrated. Accuracy regimes are included to state which thresholds are
    scientifically meaningful; non-vacuity additionally requires the regime and
    threshold map to exclude trivially achievable targets. The induced coupling
    then has three checks: admissibility and protocol validity, resource
    calibration, and objective transfer. The definition makes transcript
    simulation and the resource cost operational. Collectively,
    these elements establish a reduction theory for machine learning paradigms,
    analogous in spirit to the role of reductions in complexity theory.

    \item \textbf{Consequence and structural properties of POR reduction.}
    Section~\ref{sec:structural-results} develops the main theoretical consequences of the POR reduction framework. Proposition~\ref{prop:budget-calibration} characterizes how declared resource budgets are calibrated when a simulation consumes a bounded number of charged units. Theorem~\ref{thm:reduction-domination} establishes the central result of the framework: a POR reduction implies minimax declared-resource-complexity domination on embedded comparison classes. Building on this theorem, Corollaries~\ref{cor:lower-bound-transfer},
    \ref{cor:exact-domination}, and~\ref{cor:reduction-obstruction} provide lower-bound transfer, exact-domination results, and a complexity obstruction for ruling out proposed reductions. Beyond complexity transfer, we establish several structural properties of the framework. We prove monotonicity with respect to thresholds and comparison classes (Propositions~\ref{prop:threshold-monotonicity} and~\ref{prop:class-monotonicity}), show that POR hardness forms a preorder whose induced equivalence classes capture mutual reducibility (Proposition~\ref{prop:preorder} and Corollary~\ref{cor:equivalence-classes}), and derive reusable proof templates based on special-case embeddings, objective strengthening, additional information, and enlarged learner resources (Proposition~\ref{prop:generic-templates}). Finally, Propositions~\ref{prop:strict-objective} and~\ref{prop:scalable-objective-gap} demonstrate that modifying the learning objective alone can strictly increase minimax sample complexity, even when the environment class, observation protocol, learner class, hypothesis family, and resource accounting rule remain unchanged. Together, these results establish POR as a mathematically structured reduction theory and clarify how protocol, objective, and resource choices shape the relative difficulty of learning paradigms.

    \item \textbf{POR Structure among representative learning paradigms.}
     Section~\ref{sec:example-hardness-relations} instantiates the POR framework on abstractions of supervised, transfer, continual, and meta-learning. These results are intentionally canonical sanity checks rather than unconditional or surprising orderings: their purpose is to verify, one obligation at a time, what the stated interface-compatibility assumptions imply. For each fixed \(K\), we establish the exact reduction relations
\[
\mathrm{Cont}_{K+1} \succeq_{\mathrm{ex}} \mathrm{Trans}_{K}
\succeq_{\mathrm{ex}} \mathrm{Sup},
\qquad
\mathrm{Meta} \succeq_{\mathrm{ex}} \mathrm{Sup},
\]
    where \(A \succeq_{\mathrm{ex}} B\) denotes that \(B\) exact-reduces to \(A\). These results formally identify transfer learning as a special case of continual learning, supervised learning as a special case of transfer learning, and supervised learning as a special case of meta-learning under the stated interface assumptions and aligned raw-example costs. The reductions are established in Propositions~\ref{prop:transfer-dominates-supervised}--\ref{prop:meta-dominates-supervised}, and their sample-complexity consequences are summarized in Corollary~\ref{cor:canonical-sample-hierarchy}. We then give two worked uses: two accounting variants of the same episodic meta interface differ by the explicit overhead \(s+u\), and a standard VC-dimension lower bound propagates to the embedded transfer and continual classes. This illustrates how POR reductions enable the systematic transfer of theoretical understanding across learning paradigms.
    
\end{itemize}

\paragraph{Scope.}
The reductions studied here are structural worst-case statements. They do not
claim that optimization difficulty, constants, empirical benchmark behavior, or
typical-case performance are identical across paradigms. They also do not impose
a universal total order over supervised, transfer, continual, and meta-learning.
Instead, the framework identifies the precise assumptions under which one
learning interface contains another for minimax declared-resource-complexity
purposes, and
it explains which obligation fails when such a comparison cannot be justified.

%% file: paper/preliminary.tex
\section{Preliminaries: POR Objects and Reductions}
\label{sec:preliminaries-framework}

We use \(\mathbb N=\{0,1,2,\dots\}\), and \(\log\) denotes the natural
logarithm. All paradigms compared in a reduction share an input space
\(\mathcal X\), output space \(\mathcal Y\), base loss
\(\ell:\mathcal Y\times\mathcal Y\to\mathbb R_+\), and hypothesis family
\[
\mathcal H=\{h_\theta:\theta\in\Theta\}.
\]
For a distribution \(D\) over \(\mathcal X\times\mathcal Y\), write
\[
R_D(h):=\mathbb E_{(x,y)\sim D}[\ell(h(x),y)] .
\]
The shared \(\mathcal H\) assumption isolates protocol, objective, and
resource effects from representational capacity. The default resource unit is
one revealed labeled example, unless an explicit conversion rule is stated.

\subsection{POR objects}

A learning paradigm over \(\mathcal H\) is a tuple
\[
P=(\mathcal E_P,\mathcal O_P,\mathfrak T_P,\mathbf J_P,\kappa_P).
\]
Here \(\mathcal E_P\) is the environment class; \(\mathcal O_P\) is the
observation protocol; \(\mathfrak T_P\) is the admissible learner class;
\(\mathbf J_P\) is the performance functional; and \(\kappa_P\) is the
resource accounting rule. Environments may be single distributions, source and
target task tuples, sequential task streams, or task distributions for
episodic meta-learning.

{\color{black}
\paragraph{Operational execution law.}
The protocol has a measurable full-execution space \(\mathcal Z_P\). An element
of \(\mathcal Z_P\) records the learner-visible transcript, protocol-generated
latent evaluation coordinates, actions, learner private randomness, final
output, and any auxiliary randomness used by the performance evaluation; write
\(\pi_P^{\mathrm{vis}}:\mathcal Z_P\to\mathcal V_P\) for its projection onto
the learner-visible transcript space \(\mathcal V_P\).
The unknown environment itself is an external argument, not a coordinate that a
learner or compiler must emit. Likewise, a protocol-generated latent coordinate
such as a sampled meta-evaluation task is supplied by the environment side of
the operational law and need not be learner-visible.
For every environment \(e\), admissible learner \(T\), and external budget
\(M\), the interaction of \(T\) with \(\mathcal O_P\) induces a random execution
\(Z_{P,T,e}^M\in\mathcal Z_P\). The accounting rule is a measurable cumulative cost map
\[
\kappa_P:\mathcal Z_P\to\mathbb N.
\]
Its value is generated by an adapted, nondecreasing prefix-cost process: at
every protocol step, the charged cost so far is measurable from the current
execution prefix, cannot decrease when that prefix is extended, and equals
\(\kappa_P\) on the terminal full record. It depends only on the charged
protocol events recorded in the execution;
private and evaluation randomness is uncharged unless explicitly stated.
An \(M\)-budget execution admits a charged protocol event only if the resulting
prefix cost remains at most \(M\), and is stopped before the next charged event
that would violate
\[
\kappa_P(Z_{P,T,e}^M)\le M
\qquad\text{almost surely}.
\]
The performance functional is a measurable map of the environment and full
record; with the same symbol for the induced random variable,
\[
\mathbf J_P:\mathcal E_P\times\mathcal Z_P\to\mathbb R^{d_P},
\qquad
\mathbf J_P(T,e;M):=\mathbf J_P(e,Z_{P,T,e}^M).
\]
An admissible learner is a measurable randomized, non-anticipating procedure: each action may
depend only on the revealed history, and the learner must output an element of
\(\mathcal H\) through the interface specified by \(\mathcal O_P\). Structural
restrictions such as replay capacity, adaptation depth, or communication
topology may additionally be encoded in \(\mathfrak T_P\). Thus \(\kappa_P\)
controls the charged execution cost, whereas \(\mathfrak T_P\) controls which
procedures are allowed. Changing \(\kappa_P\) can change the transcript
available at the same numerical budget and hence can change both success and
the reduction relation.
}

\paragraph{Example: supervised learning.}
For ordinary supervised learning, \(\mathcal E_P\) may be a class of
distributions \(D\) on \(\mathcal X\times\mathcal Y\);
\(\mathcal O_P\) reveals i.i.d.\ labeled samples from \(D\);
\(\mathfrak T_P\) contains the training procedures allowed to output a
predictor \(h\in\mathcal H\); the performance functional can be the scalar risk
\[
\mathbf J_P(T,D;M)=R_D(h_T(S_M)),
\qquad S_M\sim D^M,
\]
where \(D\) is the environment, \(S_M\) is the random training sample generated
by the protocol, and \(h_T(S_M)\) is the predictor output by \(T\). Here
\(\pi_P^{\mathrm{vis}}(Z)=S\in
\bigcup_{m\ge0}(\mathcal X\times\mathcal Y)^m\) and
\(\kappa_P(Z)=|S|\), so an \(M\)-budget execution reveals at most \(M\)
labeled examples. The remaining components of \(Z\) record the learner's
randomness, output, and population-risk evaluation. Thus
\(\mathbf J_P(T,D;M)\) is random through the sampled dataset \(S_M\), but its
explicit environment argument is the distribution \(D\). The resource rule
\(\kappa_P\) counts the number of labeled examples used.

\noindent\emph{Description and purpose.} This definition describes a learning
paradigm as the full interface between an environment, a learner, an
evaluation rule, and a resource model. We need this object because the paper
compares paradigms while holding the hypothesis family fixed; the tuple makes
the protocol, objective, learner constraints, and accounting convention
explicit rather than hiding them behind a paradigm label.

\begin{assumption}[Budget capping]
\label{ass:budget-capping}
{\color{black}
Learner classes are closed under budget capping. For any admissible learner
\(T\in\mathfrak T_P\) and any budget-capping rule
\(q:\mathbb N\to\mathbb N\) satisfying \(q(M)\le M\), define \(T^q\) to be the
learner that, under external budget \(M\), runs \(T\) with internal budget
\(q(M)\), requests no further charged interaction, and leaves the remaining
allowance unused. Then \(T^q\in\mathfrak T_P\). Moreover, there is a canonical
environment-independent causal cap-wrapper coupling such that, for every
environment \(e\in\mathcal E_P\) and budget \(M\in\mathbb N\),
\[
\kappa_P(Z_{P,T^q,e}^{M})\le q(M)
\]
and
\[
Z_{P,T^q,e}^{M}=Z_{P,T,e}^{q(M)},
\qquad
\mathbf J_P(T^q,e;M)=\mathbf J_P(T,e;q(M))
\quad\text{almost surely}.
\]
The equality concerns the active full-execution record; the unused external
allowance is not part of \(\mathcal Z_P\).
}
\end{assumption}

\noindent\emph{Description and reasonableness.} This assumption states
that a declared resource budget is an upper bound: a valid learner may use fewer
charged units and leave the remaining allowance unused. If a simulation costs
only \(\psi(M)\le aM+b\) units while the reduction grants budget \(aM+b\), the
compiled learner must be allowed to run that simulation and stop. Without this
closure, reductions would depend on the unnatural distinction between
``exactly \(M\) charged units'' and ``at most \(M\) charged units.''

The induced performance variable is an \(\mathbb R^{d_P}\)-valued random vector
\(\mathbf J_P(T,e;M)\), whose law includes environment, sampling, protocol, and
learner randomness. Its coordinates may include risk, regret, forgetting, or
post-adaptation query error. For \(u,v\in\mathbb R^d\), write \(u\preceq v\)
when \(u_i\le v_i\) for every coordinate. We assume all events
\(\{\mathbf J_P(T,e;M)\preceq\tau\}\) are measurable.

\subsection{Success and minimax declared-resource complexity}

An accuracy regime for \(P\) is a set
\(\mathscr T_P\subseteq\mathbb R^{d_P}\) of meaningful threshold vectors.
Regimes are part of every reduction statement. For a reduction
\(B\preceq_{\mathrm{POR}} A\), the threshold map sends thresholds for the
solver paradigm \(A\) to thresholds for the reduced paradigm \(B\). Regimes
help expose vacuous reductions, but do not by themselves prevent them: a
scientifically meaningful reduction must also exclude trivial zero-cost
thresholds or otherwise calibrate the threshold map.

\noindent\emph{Description and purpose.} This definition describes the set of
accuracy thresholds that are considered meaningful for a paradigm. We need the
regime because reductions compare success requirements, and without restricting
the thresholds one could obtain formally correct but vacuous comparisons by
mapping every requirement to an automatically satisfied target. The
    vacuity result in Proposition~\ref{prop:vacuity-without-regimes} makes the required
nontriviality condition explicit.

{\color{black}
For a nonempty comparison class
\(\mathcal C_P\subseteq\mathcal E_P\), define the uniform fixed-budget
complexity of learner \(T\) by
\[
N_P(T,\mathcal C_P;\tau_P,\delta)
:=
\inf\Bigl\{M\in\mathbb N:
\forall e\in\mathcal C_P,\;
\Pr[\mathbf J_P(T,e;M)\preceq\tau_P]\ge1-\delta
\Bigr\},
\]
and define the minimax complexity in the declared accounting unit by
\begin{align}
N_P^\star(\mathcal C_P;\tau_P,\delta)
&:=\inf_{T\in\mathfrak T_P}
N_P(T,\mathcal C_P;\tau_P,\delta)\\
&=\inf\Bigl\{M\in\mathbb N:\exists T\in\mathfrak T_P\text{ such that }
\forall e\in\mathcal C_P,\;
\Pr[\mathbf J_P(T,e;M)\preceq\tau_P]\ge1-\delta\Bigr\}.
\label{eq:uniform-minimax}
\end{align}
We use \(\inf\emptyset=\infty\). The quantifiers require one learner and one
common budget to work over the entire comparison class; the successful budget
may not depend on the environment. This is the fixed-budget quantity used in
the objective-separation proofs below.
\emph{When \(\kappa_P\) charges labeled examples, this is sample complexity;
when it charges completed episodes, it is episode complexity.}
}
All affine expressions involving \(N_P^\star\) use the usual extended-real
conventions. Comparison classes matter: unrelated transfer tasks, shared
representation tasks, and slowly drifting continual streams are different
statistical problems.

\noindent\emph{Description and purpose.} These definitions describe the
uniform declared-resource complexity of a fixed learner over a comparison
class, and then the minimax declared-resource complexity of the best admissible
learner on that class. POR reductions are evaluated by their consequences for
worst-case complexity in the stated accounting units on explicitly chosen
environment classes.

\subsection{Reduction between learning paradigms}

The central idea is that paradigm \(B\) reduces to paradigm \(A\) if any
algorithm that succeeds under \(A\) can be systematically transformed into an
algorithm that succeeds under \(B\), with controlled change in the success
threshold and declared resource budget.

{\color{black}
\paragraph{Causal simulation kernels.}
For a protocol \(P\), let \(\mathcal F_t^P\) be the operational filtration
generated by the learner-visible projection of the full execution through
protocol step \(t\), together with all learner and currently running compiler
private randomness drawn through that step. Thus \(\mathcal F_t^P\) is exactly
the information available to the running compiled procedure; it contains no
unrevealed environment-side latent coordinate. A compiler that
simulates \(A\) from \(B\) must implement, for each \(T_A\) and \(M\), a
sequence of measurable stochastic kernels
\[
\mathsf S^{A\leftarrow B}_{T_A,M,t}
\bigl(dz_{A,t}\mid z_{A,<t},z_{B,\le \sigma_t}\bigr),
\qquad t\ge 1,
\]
where \(z_{P,t}\) denotes a simulator-controlled increment of the augmented
full-execution record and \((\sigma_t)_{t\ge1}\) is an adapted nondecreasing
schedule of almost-surely finite stopping times for the filtration
\((\mathcal F_s^B)_{s\ge0}\), so
\(0\le\sigma_1\le\sigma_2\le\cdots<\infty\) almost surely. At an interaction
step, the kernel emits the next
part of the virtual learner-visible \(A\)-transcript and the associated actions
and private randomness using only the previous augmented \(A\)-history, the
learner-visible \(B\)-history revealed by time \(\sigma_t\), and fresh private
randomness. A terminal simulator-controlled increment records the output and
fresh auxiliary evaluation randomness, but not an unknown environment or latent
task identity. Environment-generated response and latent-evaluation coordinates
are supplied by the corresponding protocol side of the joint operational law;
they cannot influence simulated learner actions except through the
learner-visible responses already revealed.
Equivalently, under the induced coupling, causality requires the virtual
operational filtration to satisfy
\[
\mathcal F_t^{A,\mathrm{virtual}}
\subseteq
\mathcal F_{\sigma_t}^{B}
\qquad\text{for every }t,
\]
where the right-hand side includes the running compiler's private randomness as
defined above.
The stopping rules and kernels may depend on \(T_A\) and \(M\), but have no
environment argument and are the same for every \(e\); the environment affects
simulator-controlled coordinates only through responses supplied by the
\(B\)-protocol. It may separately determine the external argument of
\(\mathbf J_P\) and the environment-side coordinates prescribed by the
operational protocols, neither of which is available to simulated learner
actions unless revealed through the visible transcript.

Running these kernels online with the \(B\)-protocol, the compiled learner, and
the environment-side coordinates prescribed by the embedded operational
protocol therefore induces a joint law of the two full execution records. We call this law a
\emph{protocol-valid simulation coupling} if its \(B\)-marginal is the
operational \(B\)-execution law and its virtual \(A\)-marginal is the
operational \(A\)-execution law in the embedded environment. Thus a simulation
coupling is generated by an environment-independent causal kernel implemented
by the compiler; it is not an arbitrary coupling chosen separately for each
environment. The environment-side coordinates are fixed or sampled by the
stated protocols as part of this joint law, not emitted by the compiler and not
recoupled after the executions. All other randomness shared across the
executions must arise from the online construction; no additional
environment-dependent recoupling is allowed at evaluation time. Event inclusion
below is understood under this induced law.
Equality of coupled performance variables means equality on this law; when only
equality in distribution is asserted, we write \(\overset{d}{=}\).
}

\begin{definition}[Protocol--objective--resource reduction]
Let \(A\) and \(B\) be learning paradigms over the common hypothesis family
\(\mathcal H\), with accuracy regimes
\(\mathscr T_A\subseteq\mathbb R^{d_A}\) and
\(\mathscr T_B\subseteq\mathbb R^{d_B}\). We say that \(B\)
\emph{POR-reduces} to \(A\) relative to
\((\mathscr T_A,\mathscr T_B)\), written
\[
B \preceq_{\mathrm{POR}} A,
\]
if there exist:
\begin{enumerate}
    \item an environment embedding
    \[
    \Phi : \mathcal E_B \to \mathcal E_A,
    \]
    \item a compiler
    \[
    \mathfrak C : \mathfrak T_A \to \mathfrak T_B,
    \]
    \item a threshold transfer map
    \[
    \Gamma : \mathscr T_A \to \mathscr T_B,
    \]
    \item integer overhead parameters \(a\in\mathbb N\) and
    \(b\in\mathbb N\), with \(a\ge 1\),
\end{enumerate}
{\color{black}
such that, for every learner \(T_A \in \mathfrak T_A\) and every budget
\(M \in \mathbb N\), the compiler implements an
environment-independent causal simulator
\(\mathsf S^{A\leftarrow B}_{T_A,M}\) as above. For every environment
\(e \in \mathcal E_B\), running this simulator with the \(B\)-protocol induces
a protocol-valid simulation coupling of
\[
Z_A:=Z_{A,T_A,\Phi(e)}^M,
\qquad
Z_B:=Z_{B,\mathfrak C(T_A),e}^{aM+b}
\]
with the following properties:
\begin{enumerate}
    \item \textbf{admissibility and protocol validity:}
    \(\mathfrak C(T_A)\in\mathfrak T_B\), and both coupled executions have
    their stated protocol marginals under the simulator-induced law;
    \item \textbf{resource calibration:}
    \[
    \kappa_A(Z_A)\le M,
    \qquad
    \kappa_B(Z_B)\le aM+b
    \quad\text{almost surely};
    \]
    \item \textbf{objective transfer:} for every
    \(\tau_A\in\mathscr T_A\),
    \[
    \{\mathbf J_A(T_A,\Phi(e);M)\preceq\tau_A\}
    \subseteq
    \{\mathbf J_B(\mathfrak C(T_A),e;aM+b)
       \preceq\Gamma(\tau_A)\}
    \]
    up to a null event under the simulation coupling.
\end{enumerate}
}
\end{definition}

\noindent\emph{Description and purpose.} This definition describes when one
learning paradigm can be simulated through another: environments are embedded,
learners are compiled, thresholds are translated, and declared resource budgets are
calibrated. The coupling makes protocol simulation and charged cost
checkable rather than leaving them implicit in a marginal success statement.
The reduction therefore has four pieces of data, while the induced coupling has
three checks: its first check jointly certifies compiler admissibility and both
protocol marginals.
We need this definition because informal claims that one paradigm
contains another, or that a method transfers between paradigms, are ambiguous
unless these four data components and the three coupling checks are explicit.

When the regimes are clear from context, we omit them from the notation. Exact
reductions, hardness notation, and the transfer theorems below are always
understood relative to the same fixed regimes. The word ``embedding'' denotes
the environment map used by the reduction; injectivity is not required unless
explicitly stated.

\begin{remark}
The compiler \(\mathfrak C\) is required to transform \emph{any} successful
learner for \(A\) into a learner for \(B\). This uniformity is what makes the
reduction useful for proving hardness. It is not enough to show that some
algorithm solves \(A\) and some possibly unrelated algorithm solves \(B\).
The present definition is extensional at the declared-resource-complexity level; an
efficient method-transfer claim requires a separately stated computational
bound on the compiler.
\end{remark}

\begin{remark}[Why the accuracy regime matters]
    The definition packages four nontrivial pieces of reduction data: an embedding of
environments, a simulation of admissible interaction, a threshold transfer map,
and a calibrated declared-resource overhead. In applications, most proposed hardness
comparisons fail because at least one of these ingredients is unavailable. The
accuracy regimes are part of these obligations. Without them, one could map
every target threshold to an uninformatively loose threshold and obtain a
formally correct but meaningless comparison. Even with regimes, the image
of \(\Gamma\) must exclude thresholds that are uniformly achievable at zero
cost.
\end{remark}

\begin{definition}[Exact POR-reduction]
A POR-reduction is called \emph{exact} if \(a=1\) and \(b=0\). In this case we write
\[
B \preceq_{\mathrm{ex}} A.
\]
Thus, exact reductions preserve the declared resource budget exactly.
\end{definition}

\noindent\emph{Description and purpose.} This definition describes the
zero-overhead case of a POR reduction, where the simulation preserves the
declared resource budget exactly. We need it to distinguish true
special-case containment from reductions that require additional
charged units through an affine overhead.

\begin{definition}[Relative hardness]
We say that \(A\) is at least as hard as \(B\), written
\[
A \succeq_{\mathrm{POR}} B,
\]
if
\[
B \preceq_{\mathrm{POR}} A.
\]
Similarly, \(A \succeq_{\mathrm{ex}} B\) means \(B \preceq_{\mathrm{ex}} A\).
\end{definition}

\noindent\emph{Description and purpose.} This definition describes the hardness
ordering induced by POR reductions: \(A\) is at least as hard as \(B\) when a
solver for \(A\) can be used to solve \(B\). We need this reverse-direction
notation because the reduction relation points from the simulated problem to
the solver problem, while hardness comparisons are usually read in the
opposite direction.

%% file: paper/main_results.tex
\section{Structural Results of the Framework}
\label{sec:structural-results}

The preliminary setup of Section~\ref{sec:preliminaries-framework} identifies
the basic objects that define a learning paradigm and the POR-reduction notion
used to compare paradigms. We now prove the main structural consequences
of that definition: reductions transfer upper bounds in the declared accounting
units, transfer lower bounds in the reverse direction, and generate
obstructions to proposed comparisons. Under the default labeled-example cost,
these are sample-complexity statements.

\subsection{Declared-resource-complexity transfer}

The first result formalizes a useful budget-calibration step for
POR-reductions.

\begin{proposition}[Budget calibration]
\label{prop:budget-calibration}
Fix accuracy regimes \(\mathscr T_A\) and \(\mathscr T_B\), and suppose the
reduced paradigm \(B\) satisfies the budget-capping closure in
Assumption~\ref{ass:budget-capping}. Suppose there exist an environment
embedding \(\Phi\), a pre-compiler
\(\mathfrak C_0:\mathfrak T_A\to\mathfrak T_B\), a threshold transfer map
\(\Gamma:\mathscr T_A\to\mathscr T_B\), and a budget-conversion map
\(\psi:\mathbb N\to\mathbb N\) with the following property. For every
\(T_A\in\mathfrak T_A\) and \(M\in\mathbb N\), the pre-compiler implements
an environment-independent causal simulator
\(\mathsf S^{0,A\leftarrow B}_{T_A,M}\). For every
\(e\in\mathcal E_B\), running this simulator with internal \(B\)-budget
\(\psi(M)\) induces a protocol-valid simulation coupling of
\[
Z_A:=Z_{A,T_A,\Phi(e)}^M
\quad\text{and}\quad
Z_B^0:=Z_{B,\mathfrak C_0(T_A),e}^{\psi(M)}
\]
such that \(\kappa_B(Z_B^0)\le\psi(M)\) almost surely and, for every
\(\tau_A\in\mathscr T_A\),
\[
\{\mathbf J_A(T_A,\Phi(e);M)\preceq\tau_A\}
\subseteq
\{\mathbf J_B(\mathfrak C_0(T_A),e;\psi(M))
  \preceq\Gamma(\tau_A)\}
\]
up to a null event under that simulator-induced coupling.
If there exist integer overhead parameters \(a\in\mathbb N\) and
\(b\in\mathbb N\), with \(a\ge 1\), such that
\[
\psi(M)\le aM+b
\qquad \text{for all } M\in\mathbb N,
\]
then \(B \preceq_{\mathrm{POR}} A\) with overhead parameters \(a,b\).
Budget capping promotes the same causal simulation from internal cost
\(\psi(M)\) to external cost \(aM+b\).
\end{proposition}

\noindent\emph{Description and significance.} This result shows that a
simulation whose internal charged-resource use is bounded by \(aM+b\) can be converted
into a valid POR-reduction by capping the reduced learner's budget. It is useful
because reduction proofs can first build a natural simulation and then certify
its resource overhead separately.

\begin{remark}[Episode-to-example conversion]
Proposition~\ref{prop:budget-calibration} is especially useful when two
paradigms count different atomic units. If one unit of paradigm \(A\) can always
be implemented using at most \(q\) units of paradigm \(B\) on the embedded
class, then \(\psi(M)=qM\) and the resulting POR-reduction has multiplicative
overhead \(a=q\), \(b=0\).
\end{remark}

{\color{black}
\begin{proposition}[Episode-to-example accounting can be genuinely non-exact]
\label{prop:episode-example-accounting}
Fix support and query sizes \(s,u\in\mathbb N\) with \(s,u\ge1\), and set
\(r:=s+u\). Consider two versions of the same episodic meta-learning
experiment, with identical environments, raw full-execution law, output
interface, and performance functional. Their common learner class is closed
under the identity wrapper between the two accounting interfaces, and
performance is a measurable function of the shared raw execution and output;
only the accounting/stopping rule changes. Equip both paradigms with the
singleton accuracy regime \(\{0\}\). In
\(P_{\mathrm{epi}}^{s,u}\), a completed support/query episode costs one unit;
in \(P_{\mathrm{ex}}^{s,u}\), every labeled support or query example costs one
unit. Both protocols expose only completed episodes. Then
\[
P_{\mathrm{ex}}^{s,u}
\preceq_{\mathrm{POR}}
P_{\mathrm{epi}}^{s,u}
\]
with overhead \((a,b)=(r,0)\). Moreover, exactness can fail under the canonical
common learner class of all randomized causal policies from completed-episode
transcripts to \(\mathcal H\): there is a two-environment, zero-one-loss
comparison class \(\mathcal C\) on which, for
every \(\delta\in(0,1/2)\),
\[
N_{P_{\mathrm{epi}}^{s,u}}^\star(\mathcal C;0,\delta)=1,
\qquad
N_{P_{\mathrm{ex}}^{s,u}}^\star(\mathcal C;0,\delta)=r,
\]
and there is no exact POR-reduction from the example-accounted restriction on
\(\mathcal C\) to the episode-accounted restriction on \(\mathcal C\).
Consequently some nonzero overhead is necessary on this restriction. The
statement does not claim that \((r,0)\) is a minimal affine-overhead pair.
\end{proposition}

\noindent\emph{Worked interpretation.}
At episode budget \(M\), the solver sees \(M\) complete episodes. Reproducing
the identical transcript when examples are charged individually costs exactly
\((s+u)M\). For instance, five support and fifteen query examples per episode
give the explicit conversion \(M\mapsto20M\). The identity environment and
threshold maps, the identity learner wrapper, equality of the coupled
transcripts, and this cost calculation provide the four pieces of reduction
data and verify all three coupling checks. The
detailed proof is in Appendix~\ref{app:proof-episode-accounting}.
}

\begin{theorem}[Reduction implies declared-resource-complexity domination]
\label{thm:reduction-domination}
Fix accuracy regimes \(\mathscr T_A\) and \(\mathscr T_B\). Suppose
\(B \preceq_{\mathrm{POR}} A\) relative to these regimes, with embedding
\(\Phi\), threshold map \(\Gamma\), and overhead parameters \(a,b\). Then for
every comparison class \(\mathcal C_B \subseteq \mathcal E_B\), every threshold
\(\tau_A \in \mathscr T_A\), and every confidence level \(\delta \in (0,1)\),
\[
N_B^\star(\mathcal C_B;\Gamma(\tau_A),\delta)
\;\le\;
a\,N_A^\star(\Phi(\mathcal C_B);\tau_A,\delta)+b,
\]
where
\[
\Phi(\mathcal C_B):=\{\Phi(e):e\in\mathcal C_B\}.
\]
\end{theorem}

\noindent\emph{Description and significance.} The theorem states that a
POR-reduction from \(B\) to \(A\) forces the minimax complexity, measured in
each paradigm's declared accounting unit, of the embedded \(B\)-class to be
bounded by that of \(A\), up to the reduction overhead. This is the central
soundness guarantee of the framework: a valid reduction is a simulation
certificate whose definition entails worst-case declared-resource-complexity
domination at one common budget over the comparison class. When both accounting
rules charge labeled examples, this is the usual sample-complexity domination
statement.

\begin{corollary}[Lower-bound transfer]
\label{cor:lower-bound-transfer}
Under the assumptions of Theorem~\ref{thm:reduction-domination}, for every
comparison class \(\mathcal C_B \subseteq \mathcal E_B\), every threshold
\(\tau_A\in\mathscr T_A\), and every \(\delta\in(0,1)\),
\[
N_A^\star(\Phi(\mathcal C_B);\tau_A,\delta)
\;\ge\;
\frac{N_B^\star(\mathcal C_B;\Gamma(\tau_A),\delta)-b}{a}.
\]
Hence, any lower bound for \(B\) induces a lower bound for \(A\).
\end{corollary}

\noindent\emph{Description and significance.} This corollary rewrites the main
domination inequality as a lower-bound transfer principle. Its significance is
that a hard embedded subproblem of \(B\) immediately yields a lower bound for
the richer or more demanding paradigm \(A\).

\begin{corollary}[Exact declared-resource-complexity domination]
\label{cor:exact-domination}
Under the assumptions of Theorem~\ref{thm:reduction-domination}, if
\(B \preceq_{\mathrm{ex}} A\), then for every comparison class
\(\mathcal C_B \subseteq \mathcal E_B\), every threshold
\(\tau_A\in\mathscr T_A\), and every \(\delta\in(0,1)\),
\[
N_B^\star(\mathcal C_B;\Gamma(\tau_A),\delta)
\le
N_A^\star(\Phi(\mathcal C_B);\tau_A,\delta).
\]
\end{corollary}

\noindent\emph{Description and significance.} Exact reductions remove the
additive and multiplicative calibration terms from the comparison. Thus, when
one paradigm contains another as a genuine special case, the resulting minimax
declared-resource-complexity inequality is direct and has no budget loss.

\begin{corollary}[Declared-resource-complexity obstruction to reductions]
\label{cor:reduction-obstruction}
Fix candidate reduction data consisting of an environment embedding \(\Phi\),
a threshold map \(\Gamma:\mathscr T_A\to\mathscr T_B\), and overhead
parameters \(a,b\). If there exist a comparison class \(\mathcal C_B\), a
threshold \(\tau_A\in\mathscr T_A\), and a confidence level
\(\delta\in(0,1)\) such that
\[
N_B^\star(\mathcal C_B;\Gamma(\tau_A),\delta)
>
a\,N_A^\star(\Phi(\mathcal C_B);\tau_A,\delta)+b,
\]
then no POR-reduction with this candidate data exists. In particular, if the
displayed inequality can be witnessed for every possible embedding and
threshold map, then \(B\not\preceq_{\mathrm{POR}} A\) for the specified
overhead parameters.
\end{corollary}

\noindent\emph{Description and significance.} This obstruction gives a way to
falsify proposed POR-reductions: if the declared-resource complexities violate the
inequality that a reduction would imply, that reduction cannot exist. It turns
complexity separations into diagnostic tools for checking informal
claims of containment.

\subsection{General structural properties}

The next propositions record the main structural consequences of the framework.

\begin{proposition}[Monotonicity in thresholds]
\label{prop:threshold-monotonicity}
Fix a learning paradigm \(P\), a comparison class \(\mathcal C_P\subseteq\mathcal E_P\), and \(\delta\in(0,1)\). If \(\tau_P,\tau_P' \in \mathbb R^{d_P}\) satisfy
\[
\tau_P \preceq \tau_P',
\]
then
\[
N_P^\star(\mathcal C_P;\tau_P,\delta)
\ge
N_P^\star(\mathcal C_P;\tau_P',\delta).
\]
\end{proposition}

\noindent\emph{Description and significance.} The proposition records that
weaker success requirements cannot increase minimax declared-resource complexity. This sanity check matters because POR comparisons depend on the chosen accuracy regime, and monotonicity ensures that threshold relaxations behave as expected.

\begin{proposition}[Monotonicity in comparison classes]
\label{prop:class-monotonicity}
Fix a learning paradigm \(P\), a threshold \(\tau_P\), and \(\delta\in(0,1)\). If
\[
\mathcal C_P \subseteq \mathcal C_P' \subseteq \mathcal E_P,
\]
then
\[
N_P^\star(\mathcal C_P;\tau_P,\delta)
\le
N_P^\star(\mathcal C_P';\tau_P,\delta).
\]
\end{proposition}

\noindent\emph{Description and significance.} The result states that enlarging
the comparison class cannot make a minimax learning problem easier. It is
important because claims about paradigm hardness are only meaningful relative
to the environment class over which uniform performance is required.

\begin{proposition}[Preorder structure for regime-equipped paradigms]
\label{prop:preorder}
Fix an accuracy regime for each paradigm under consideration. The relation
\(\succeq_{\mathrm{POR}}\) is a preorder on these regime-equipped paradigms.
In particular:
\begin{enumerate}
    \item it is reflexive;
    \item it is transitive.
\end{enumerate}
The same is true for \(\succeq_{\mathrm{ex}}\).
\end{proposition}

\noindent\emph{Description and significance.} This proposition shows that the
POR hardness relation is reflexive and transitive once accuracy regimes are
fixed. The result justifies chaining reductions across paradigms and treating
the framework as an organizing structure rather than a collection of isolated
comparisons.

\begin{corollary}[Equivalence classes]
\label{cor:equivalence-classes}
For regime-equipped paradigms, define \(A \equiv_{\mathrm{POR}} B\) if
\(A \succeq_{\mathrm{POR}} B\) and \(B \succeq_{\mathrm{POR}} A\). Then
\(\equiv_{\mathrm{POR}}\) is an equivalence relation, and
\(\succeq_{\mathrm{POR}}\) induces a partial order on the equivalence classes.
\end{corollary}

\noindent\emph{Description and significance.} The corollary says that mutually
reducible paradigms form equivalence classes, and that these classes are
partially ordered by POR hardness. This makes precise when two superficially
different learning setups should be treated as structurally equivalent for the
purposes of worst-case complexity in the declared accounting units.

\subsection{Generic sufficient conditions for hardness}

The next proposition gives several reusable proof templates. In practice, most cross-paradigm hardness results are instances of one of these patterns.

\begin{proposition}[Generic sufficient conditions]
\label{prop:generic-templates}
Let \(A\) and \(B\) be learning paradigms over the same hypothesis family, with
fixed accuracy regimes \(\mathscr T_A\) and \(\mathscr T_B\).
\begin{enumerate}
    \item \textbf{Special-case embedding.}
    Suppose there exist an embedding
    \(\Phi:\mathcal E_B\to\mathcal E_A\) and a compiler
    \(\mathfrak C:\mathfrak T_A\to\mathfrak T_B\) such that
    for every \(e,T,M\), \(\mathfrak C(T)\) implements an
    environment-independent causal, protocol-valid simulation coupling of the
    budget-\(M\) execution of \(T\) for \(A\) on \(\Phi(e)\) and the
    budget-\(M\) execution of \(\mathfrak C(T)\) for \(B\) on \(e\), with both
    declared costs at most \(M\) almost surely. Suppose also
    that there exists \(\Gamma:\mathscr T_A\to\mathscr T_B\) such that
    \[
    \{\mathbf J_A(T,\Phi(e);M)\preceq\tau_A\}
    \subseteq
    \{\mathbf J_B(\mathfrak C(T),e;M)\preceq \Gamma(\tau_A)\}
    \]
    under the protocol-valid simulation coupling, up to null events, for all \(e,T,M\), and all
    \(\tau_A\in\mathscr T_A\). Then
    \(A \succeq_{\mathrm{ex}} B\).

    \item \textbf{Objective strengthening.}
    Suppose \(A\) and \(B\) have the same environments, protocol, learner class, and resource accounting, and suppose there exists \(\Gamma:\mathscr T_A\to\mathscr T_B\) such that
    \[
    \{\mathbf J_A(T,e;M)\preceq\tau_A\}
    \subseteq
    \{\mathbf J_B(T,e;M)\preceq \Gamma(\tau_A)\}
    \]
    under the identity coupling, up to null events, for all \(e,T,M\), and all
    \(\tau_A\in\mathscr T_A\). Then
    \(A \succeq_{\mathrm{ex}} B\).

    \item \textbf{Information monotonicity.}
    Suppose \(A\) and \(B\) have the same environments, learner class,
    objective, resource accounting, and accuracy regime, but \(B\)'s protocol
    reveals at least as much side information as \(A\)'s protocol, and learners
    under \(B\) are free to ignore the extra information. Then
    \(A \succeq_{\mathrm{ex}} B\).

    \item \textbf{Resource monotonicity.}
    Suppose \(A\) and \(B\) have the same environments, protocol, objective,
    resource accounting, and accuracy regime, but
    \[
    \mathfrak T_A \subseteq \mathfrak T_B.
    \]
    Then \(A \succeq_{\mathrm{ex}} B\).
\end{enumerate}
\end{proposition}

\noindent\emph{Description and significance.} The proposition gives reusable
routes for proving POR-reductions through special-case embeddings, objective
strengthening, information monotonicity, and resource monotonicity. Its value is
practical: it converts common informal arguments about one paradigm containing
another into a short list of formal obligations.

\subsection{Strictness of objective-based hardness}
\label{sec:strictness}

Objective strengthening is not merely a formal inclusion: changing an objective
can strictly change minimax sample complexity even when the environment class,
protocol, learner class, and hypothesis family are fixed. The next proposition
gives a minimal example. In the two constructions below, \(h_T^M\) denotes the
predictor output by learner \(T\) after the \(M\)-sample transcript generated
by the stated i.i.d.\ protocol.

\begin{proposition}[Objective strengthening can be strict]
\label{prop:strict-objective}
Let \(\mathcal X=\{1,2\}\), \(\mathcal Y=\{0,1\}\), let \(\ell\) be the
zero-one loss, and let \(\mathcal H\) contain all functions
\(\mathcal X\to\mathcal Y\). For each
\(\sigma=(\sigma_1,\sigma_2)\in\{0,1\}^2\), let \(D_\sigma\) be the
distribution that draws \(x\) uniformly from \(\{1,2\}\) and returns the
deterministic label \(y=\sigma_x\). Consider two paradigms with the same
environment class \(\{D_\sigma:\sigma\in\{0,1\}^2\}\), the same i.i.d.
sampling protocol, and the same unrestricted learner class:
\[
\mathbf J_{\mathrm{one}}(T,D_\sigma;M)
=
\mathbf 1\{h_T^M(1)\ne \sigma_1\},
\]
and
\[
\mathbf J_{\mathrm{all}}(T,D_\sigma;M)
=
\bigl(
\mathbf 1\{h_T^M(1)\ne \sigma_1\},
\mathbf 1\{h_T^M(2)\ne \sigma_2\}
\bigr).
\]
Use the zero-threshold accuracy regimes
\(\mathscr T_{\mathrm{one}}=\{0\}\) and
\(\mathscr T_{\mathrm{all}}=\{(0,0)\}\), and let
\(\mathcal C=\{D_\sigma:\sigma\in\{0,1\}^2\}\). Then
\[
P_{\mathrm{all}}\succeq_{\mathrm{ex}} P_{\mathrm{one}},
\]
but the domination is strict: for every \(\delta\in(0,1/2)\),
\[
N_{P_{\mathrm{one}}}^\star(\mathcal C;0,\delta)
=
\left\lceil \log_2\frac{1}{2\delta}\right\rceil,
\qquad
N_{P_{\mathrm{all}}}^\star(\mathcal C;(0,0),\delta)
=
\left\lceil \log_2\frac{1}{\delta}\right\rceil .
\]
Consequently \(P_{\mathrm{all}}\) and \(P_{\mathrm{one}}\) are not equivalent
under exact POR-reductions relative to these zero-threshold regimes.
\end{proposition}

\noindent\emph{Description and significance.} This example proves that a
stronger objective can strictly increase minimax sample complexity even when
the environments, data protocol, hypothesis class, and learner class are
unchanged. It demonstrates that POR hardness is not only about information
access; the definition of success itself can create genuine statistical
difficulty.

{\color{black}
\paragraph{Uniform-overhead reduction families.}
For indexed, regime-equipped families
\(\mathbf A=(A_m)_{m\in I}\) and \(\mathbf B=(B_m)_{m\in I}\), write
\(\mathbf B\preceq_{\mathrm{uPOR}}\mathbf A\) if there exist coefficients
\(a\ge1\) and \(b\ge0\), independent of \(m\), such that for every \(m\) there
is a POR reduction \(B_m\preceq_{\mathrm{POR}}A_m\) with overhead \((a,b)\).
The embeddings \(\Phi_m\), compilers \(\mathfrak C_m\), and threshold maps
\(\Gamma_m\) may depend arbitrarily on \(m\); requiring one uniformly
computable construction would be a stronger notion not used here. Uniform
\(a,b\) express dimension-independent resource calibration; allowing them to
grow with \(m\) can absorb the asymptotic gap used below. This notion does
not rule out separately calibrated reductions with coefficients
\((a_m,b_m)\) for each fixed \(m\).
}

\begin{proposition}[A scalable objective-strengthening gap]
\label{prop:scalable-objective-gap}
For each \(m\ge 3\), let
\(\mathcal X_m=\{1,\dots,m\}\), \(\mathcal Y=\{0,1\}\), and let
\(\mathcal H_m\) contain all functions \(\mathcal X_m\to\mathcal Y\). For
each \(\sigma\in\{0,1\}^m\), let \(D_\sigma\) draw \(x\) uniformly from
\(\mathcal X_m\) and return \(y=\sigma_x\). Define two paradigms with the same
environment class, i.i.d. sampling protocol, and unrestricted learner class:
\[
\mathbf J_{\mathrm{one}}^{(m)}(T,D_\sigma;M)
=
\mathbf 1\{h_T^M(1)\ne \sigma_1\},
\]
and
\[
\mathbf J_{\mathrm{all}}^{(m)}(T,D_\sigma;M)
=
\bigl(\mathbf 1\{h_T^M(j)\ne \sigma_j\}\bigr)_{j=1}^m.
\]
Let \(\mathcal C_m=\{D_\sigma:\sigma\in\{0,1\}^m\}\), use the singleton
accuracy regimes
\[
\mathscr T_{\mathrm{one}}^{(m)}=\{0\},
\qquad
\mathscr T_{\mathrm{all}}^{(m)}=\{\mathbf 0_m\},
\]
where \(\mathbf 0_m\) is the zero vector in \(\mathbb R^m\), and set
\(\delta=1/8\). Then
\[
N_{P_{\mathrm{one}}^{(m)}}^\star(\mathcal C_m;0,\delta)
\le
2m,
\qquad
N_{P_{\mathrm{all}}^{(m)}}^\star(\mathcal C_m;\mathbf 0_m,\delta)
\ge
\frac{m}{4}\log m .
\]
Consequently,
\[
\bigl(P_{\mathrm{all}}^{(m)}\bigr)_{m\ge3}
\not\preceq_{\mathrm{uPOR}}
\bigl(P_{\mathrm{one}}^{(m)}\bigr)_{m\ge3}.
\]
This conclusion rules out only dimension-independent affine overhead; it does
not rule out a separately calibrated reduction for a fixed \(m\).
\end{proposition}

\noindent\emph{Description and significance.} This scalable example amplifies
the previous strictness result by producing a growing gap between one-coordinate
and all-coordinate objectives. The growth rules out a uniform affine-overhead
reverse reduction, showing that objective strengthening prevents
dimension-independent resource calibration across this indexed family.

\begin{remark}[Exact versus affine overhead]
Proposition~\ref{prop:strict-objective} rules out exact reverse reductions in
a minimal two-point example, while
Proposition~\ref{prop:scalable-objective-gap} rules out reverse reductions with
any fixed affine overhead across the scalable family. The embeddings,
compilers, and threshold maps may still vary with \(m\); only \(a,b\) are
shared. This distinction matters:
constant sample gaps are meaningful for exact containment, but affine-overhead
POR reductions intentionally ignore fixed additive and multiplicative
calibration costs.
\end{remark}

\section{Hardness Relations Among Representative Paradigms}
\label{sec:example-hardness-relations}

We now instantiate the general framework on four representative paradigms:
supervised learning, transfer learning, continual learning, and meta-learning.
The purpose of this section is not to impose a universal total ordering on all
paradigms, but rather to identify the hardness relations that follow directly
from the framework under natural closure assumptions. The exact relations
below are deliberately assumption-driven sanity checks. They show that the POR
definition recovers familiar special cases only after environment membership,
compiler admissibility, transcript-law compatibility, objective alignment, and
resource accounting have all been verified; they are not presented as an
unconditional or surprising hierarchy.

\subsection{Canonical forms of the representative paradigms}

We use the following simplified canonical models. For every fixed
\(K\ge0\), \(\mathrm{Trans}_K\) has exactly \(K\) source phases and one target
phase; \(\mathrm{Cont}_{K+1}\) has exactly \(K+1\) sequential phases. Unindexed
names are shorthand only, and every displayed reduction below is between
fixed-dimensional indexed objects. For each external budget \(M\), the
protocol specifies a deterministic phase-allocation vector
   \(\mathbf m_K(M)=(m_1(M),\ldots,m_{K+1}(M))\) with
   \(\sum_jm_j(M)=M\). A learner using an internal cap \(q(M)\) declares that
   cap before phase allocation and receives the schedule
   \(\mathbf m_K(q(M))\). Transcripts contain zero-cost phase and task markers,
and every revealed labeled example costs one unit. The canonical learner class
contains all randomized measurable non-anticipating procedures compatible with
the stated interface and is closed under the causal tagging, relabeling, and
reblocking maps explicitly specified below. These conventions make the learner
classes, protocol law, and accounting unit part of the canonical object rather
 than implicit proof assumptions. For any canonical learner \(T\) run at external
 budget \(M\), write \(q_T(M)\le M\) for its preallocation cap declaration, with
 \(q_T(M)=M\) when no smaller cap is declared. Each compiler below obtains and
 relays \(q_T(M)\) before the reduced protocol performs any allocation or charged
 interaction, and then uses the schedule or reblocking map indexed by
 \(q_T(M)\). Throughout this subsection, \(h_{T,e}^M\)
denotes the predictor output by \(T\) after the budget-\(M\) interaction with
environment \(e\); when the environment is a single distribution \(D\), we also
write \(h_{T,D}^M\).

\paragraph{Supervised learning.}
An environment is a single distribution \(D\) over \(\mathcal X\times\mathcal Y\). The protocol reveals i.i.d.\ labeled samples from \(D\). The performance functional is
\[
\mathbf J_{\mathrm{Sup}}(T,D;M)=R_D(h_{T,D}^M),
\]
where \(h_{T,D}^M\in\mathcal H\) is the predictor produced by learner \(T\) after budget \(M\). This canonical form follows the standard distributional setup in statistical learning theory
\citep{vapnik1998statistical,shalev2014understanding,mohri2018foundations}.

\paragraph{Transfer learning.}
For \(\mathrm{Trans}_K\), an environment is a tuple
\[
e=(D_1,\dots,D_K;D_\star),
\]
where \(D_1,\dots,D_K\) are source tasks and \(D_\star\) is the target task. The protocol reveals source samples first and target samples afterward. The performance functional is
\[
\mathbf J_{\mathrm{Trans}_K}(T,e;M)=R_{D_\star}(h_{T,e}^M).
\]
This abstraction captures the source-target viewpoint common in transfer and domain-adaptation theory
\citep{pan2010survey,bendavid2010theory,blitzer2008learning,mansour2009domain,kalan2020minimax}.

\paragraph{Continual learning.}
For \(\mathrm{Cont}_{K+1}\), an environment is a sequential task stream
\[
e=(D_1,\dots,D_K,D_{K+1}).
\]
The protocol reveals tasks sequentially. We use the performance functional
\[
\mathbf J_{\mathrm{Cont}_{K+1}}(T,e;M)
=
\bigl(
    R_{D_1}(h_{T,e}^M),\dots,R_{D_{K+1}}(h_{T,e}^M),F(T,e;M)
\bigr),
\]
where \(F(T,e;M)\) is a forgetting functional. This form reflects the standard emphasis on sequential exposure, replay, task information, and forgetting in continual learning
\citep{robins1995catastrophic,french1999catastrophic,parisi2019continual,lopezpaz2017gradient,delange2022continual,vanderven2022three}.

{\color{black}
\paragraph{Meta-learning.}
An environment is a distribution \(\Pi\) over tasks \(D\). The protocol fixes a
deterministic budget-indexed allocation schedule
\[
\bigl(m_{\mathrm{train}}(M),n_{\mathrm{ev}}(M)\bigr)\in\mathbb N^2,
\qquad
m_{\mathrm{train}}(M)+n_{\mathrm{ev}}(M)\le M.
\]
A learner using an internal cap \(q_T(M)\) declares it before allocation and
uses the schedule indexed by \(\bar M:=q_T(M)\), with \(\bar M=M\) when no
smaller cap is declared; zero-sized blocks are permitted at small budgets. At
external budget \(M\), the protocol first reveals a meta-training transcript
containing \(m_{\mathrm{train}}(\bar M)\) charged labeled examples. It
then samples a latent evaluation task \(D_{\mathrm{ev}}\sim\Pi\), reveals only
an anonymous evaluation-phase marker, and reveals an evaluation support set
\(S_{\mathrm{ev}}\sim D_{\mathrm{ev}}^{n_{\mathrm{ev}}(\bar M)}\). Thus
\[
m_{\mathrm{train}}(\bar M)+|S_{\mathrm{ev}}|\le \bar M\le M.
\]
The latent evaluation task is an environment-side coordinate of the full
budgeted execution record; its support set, the meta-training transcript, and
the learner randomness are the remaining relevant coordinates. The task
identity itself is not learner-visible and is not emitted by a learner or
compiler. If
\(A_{T,\Pi}^{m_{\mathrm{train}}(\bar M)}(S_{\mathrm{ev}})\in\mathcal H\) denotes the
predictor obtained after meta-training and adaptation, define the performance
functional to be the random post-adaptation population risk
\[
\mathbf J_{\mathrm{Meta}}(T,\Pi;M)
=
R_{D_{\mathrm{ev}}}\!\left(
A_{T,\Pi}^{m_{\mathrm{train}}(\bar M)}(S_{\mathrm{ev}})
\right).
\]
Thus neither evaluation-task nor support-set randomness is averaged out before
the success event is formed. The accounting rule charges every labeled example
revealed during meta-training and evaluation support; population query risk
reveals no additional examples. Unless episode accounting is explicitly
invoked, \(M\) therefore counts all such raw examples. The support/query
abstraction is standard in learning-to-learn and few-shot meta-learning
\citep{baxter2000model,vinyals2016matching,santoro2016meta,ravi2017optimization,snell2017prototypical,finn2017model,hospedales2022meta}.
}

\begin{table}[ht]
\centering
\caption{Canonical instantiations of the four representative paradigms. The
table suppresses algorithmic details and keeps only the ingredients needed for
POR comparisons. \\} 
\label{tab:canonical-paradigms}
\begin{tabular}{@{}L{0.16\linewidth}L{0.34\linewidth}L{0.42\linewidth}@{}}
\toprule
Paradigm & Environment class & Controlled performance \\
\midrule
Supervised & single task distribution \(D\) & population risk on \(D\) \\
Transfer \(\mathrm{Trans}_K\) & \((D_1,\dots,D_K;D_\star)\) & target risk after source-then-target interaction \\
Continual \(\mathrm{Cont}_{K+1}\) & ordered task stream \((D_1,\dots,D_K,D_{K+1})\) & risks on all tasks together with forgetting \\
Meta & task distribution \(\Pi\) & random post-adaptation population risk \\
\bottomrule
\end{tabular}
\end{table}

\subsection{Closure assumptions}

{\color{black}
To derive concrete hardness statements, we use the following explicit
interface-compatibility assumptions. Fix a scalar risk regime
\(\mathscr I\subseteq\mathbb R_+\) that excludes thresholds uniformly
achievable at zero cost on the comparison classes under discussion.
Supervised, transfer, and meta-learning use \(\mathscr I\); continual learning
uses threshold vectors whose final embedded-task coordinate lies in
\(\mathscr I\). The threshold maps below are identity maps on \(\mathscr I\)
for scalar-risk reductions and final-task coordinate projection for the
transfer-to-continual reduction.

\begin{assumption}[Diagonal-task transfer-interface compatibility]
\label{ass:degenerate-transfer}
For every supervised environment \(D\), \(\mathrm{Trans}_K\) contains
\[
\Phi_{\mathrm{Sup}\to\mathrm{Trans}_K}(D)
:=(\underbrace{D,\ldots,D}_{K\text{ source phases}};D).
\]
For any predeclared cap \(\bar M\le M\), the compiled supervised wrapper relays
\(\bar M\) before requesting data, partitions
\(S_{\bar M}\sim D^{\bar M}\) according to \(\mathbf m_K(\bar M)\), and adds
the corresponding source/target phase markers. This causal tagging map has
exactly the external-budget-\(M\) \(\mathrm{Trans}_K\)-protocol law for a solver
with cap \(\bar M\) on the displayed environment. The canonical
learner classes are closed under the wrapper that feeds this tagged transcript
to an arbitrary \(T_{\mathrm{Trans}_K}\) and returns its predictor. Both
accounting rules charge the same revealed examples, and under the induced
coupling the compiled supervised predictor equals the transfer predictor; hence
their population target-risk variables agree.
\end{assumption}

\noindent\emph{Interpretation.}
This assumption uses diagonal rather than empty source tasks so that every
statement concerns one fixed \(K\)-indexed transfer object. Its environment
clause encodes the familiar special case, while its remaining clauses explicitly
supply the learner, protocol, objective, and cost compatibility that a POR proof
must verify.

\begin{assumption}[Transfer--continual interface compatibility]
\label{ass:continual-target}
For every
\(e=(D_1,\dots,D_K;D_\star)\in\mathcal E_{\mathrm{Trans}_K}\),
\[
\Phi_{\mathrm{Trans}_K\to\mathrm{Cont}_{K+1}}(e)
:=(D_1,\dots,D_K,D_\star)
\]
belongs to \(\mathcal E_{\mathrm{Cont}_{K+1}}\). For any predeclared cap
\(\bar M\le M\), the wrapper relays \(\bar M\) before phase allocation and the
two protocols use the same allocation \(\mathbf m_K(\bar M)\), task boundaries,
and task-marked sample blocks; the causal relabeling ``source \(j\)'' to task \(j\)
and ``target'' to task \(K+1\) therefore preserves the transcript law. The
canonical learner classes are closed under this relabeling wrapper, both costs
charge one unit per revealed example, the final predictors coincide under the
coupling, and the continual objective contains
\(R_{D_\star}\) as coordinate \(K+1\).
\end{assumption}

\noindent\emph{Interpretation.}
The environment-level containment is only one clause. The shared schedule,
learner closure, cost equality, and final-risk coordinate are what make the
compiler admissible and the claimed objective implication valid.

\begin{assumption}[Singleton-task episodic-interface compatibility]
\label{ass:singleton-meta}
For every supervised environment \(D\), the meta environment class contains
\(\Pi_D:=\delta_D\). When both paradigms count revealed labeled examples,
there is, for every \(\bar M\in\mathbb N\), a deterministic causal reblocking map
\[
\rho_{\bar M}:(\mathcal X\times\mathcal Y)^{\bar M}\to\mathcal V_{\mathrm{Meta}}
\]
independent of \(D\), such that
\((\rho_{\bar M})_\#D^{\bar M}\) is the required
learner-visible meta-transcript law on \(\delta_D\), including a meta-training
block of size \(m_{\mathrm{train}}(\bar M)\), an anonymous evaluation-phase marker,
and an evaluation support block \(S_{\mathrm{ev}}\) of size
\(n_{\mathrm{ev}}(\bar M)\). The map \(\rho_{\bar M}\) emits no task identity.
At external budget \(M\), the wrapper first relays the solver's predeclared cap
\(\bar M\le M\), requests only \(\bar M\) examples, and applies
\(\rho_{\bar M}\). The wrapper and the meta protocol then induce a coupling of
the full execution records in which
the environment-side latent task supplied by \(\delta_D\) satisfies
\(D_{\mathrm{ev}}=D\) almost surely,
\[
m_{\mathrm{train}}(\bar M)+|S_{\mathrm{ev}}|\le \bar M,
\]
and the meta execution has cost at most \(\bar M\).
The canonical learner classes are closed under the wrapper that buffers and
reblocks the supervised stream, runs the meta learner on the resulting
meta-training transcript, adapts it on the charged set \(S_{\mathrm{ev}}\), and
returns that adapted predictor. Under this coupling the compiled supervised
predictor is the same random predictor
\(A_{T,\delta_D}^{m_{\mathrm{train}}(\bar M)}(S_{\mathrm{ev}})\) produced by the meta
execution, so their random population-risk variables agree. Population query
evaluation reveals no additional labeled examples.
\end{assumption}

\noindent\emph{Interpretation.}
Exactness here is conditional on raw-example accounting for both meta-training
and evaluation support, together with the stated causal reblocking closure. The
evaluation support remains random inside the charged coupled execution rather
than being averaged out or supplied for free. If the meta interface instead
charges completed episodes, Proposition~\ref{prop:episode-example-accounting}
gives the corresponding non-exact conversion.
}

Table~\ref{tab:canonical-reduction-data} records the reduction data used below.
The following propositions then verify that these data satisfy the POR
definition under the stated closure assumptions.

\begin{center}
\refstepcounter{table}\label{tab:canonical-reduction-data}
\small
\begin{minipage}{0.98\linewidth}
\textbf{Table~\thetable:} Reduction data for the canonical exact reductions.
Each row lists the four pieces of POR reduction data: environment embedding
\(\Phi\), learner compiler \(\mathfrak C\), threshold map \(\Gamma\), and resource
overhead. The preceding assumptions additionally verify compiler
admissibility, the coupled transcript marginals, objective transfer, and no
declared-budget inflation. \\
\vspace{0.5ex}

\setlength{\tabcolsep}{3pt}
\begin{tabular}{@{}L{0.15\linewidth}L{0.26\linewidth}L{0.31\linewidth}L{0.13\linewidth}L{0.07\linewidth}@{}}
\toprule
Red. & Embedding \(\Phi\) & Compiler \(\mathfrak C\) & \(\Gamma\) & Cost \\
\midrule
\(\mathrm{Sup}\preceq\mathrm{Trans}_K\)
&
\(D\mapsto(D,\ldots,D;D)\)
&
Tag blocks of \(S_M\sim D^M\) as the \(K\) source phases and target phase.
&
Identity on risk
&
\((1,0)\)
\\
\(\mathrm{Trans}_K\preceq\mathrm{Cont}_{K+1}\)
&
\(\begin{aligned}[t]
&(D_1,\dots,D_K;D_\star)\\
&\mapsto(D_1,\dots,D_K,D_\star)
\end{aligned}\)
&
Relabel source/target phase tags as tasks \(1{:}K{+}1\); output the final predictor.
&
Project final-task risk
&
\((1,0)\)
\\
\(\mathrm{Sup}\preceq\mathrm{Meta}\)
&
\(D\mapsto\delta_D\)
&
Reblock a charged sample stream from \(D\) into meta-training and evaluation
support within the same declared budget.
&
Identity on post-adaptation risk
&
\((1,0)\)
\\
\bottomrule
\end{tabular}
\end{minipage}
\end{center}

\subsection{Concrete hardness relations}

The first result says that transfer learning is at least as hard as supervised learning.

\begin{proposition}[Transfer learning dominates supervised learning]
\label{prop:transfer-dominates-supervised}
Under Assumption~\ref{ass:degenerate-transfer},
\[
\mathrm{Trans}_K \succeq_{\mathrm{ex}} \mathrm{Sup}.
\]
\end{proposition}

\noindent\emph{Description and significance.} The proposition embeds supervised
learning into transfer learning by tagging one i.i.d. sample stream from
\(D\) as diagonal source and target phases. It anchors the hierarchy with an
exact special-case containment and shows that worst-case transfer learning is at
least as hard as supervised learning under the stated full-interface
assumption.

The next proposition formalizes the intuition that continual learning is at least as hard as transfer learning when the continual objective requires control of all tasks, including the final one.

\begin{proposition}[Continual learning dominates transfer learning]
\label{prop:continual-dominates-transfer}
Under Assumption~\ref{ass:continual-target},
\[
\mathrm{Cont}_{K+1} \succeq_{\mathrm{ex}} \mathrm{Trans}_K.
\]
\end{proposition}

\noindent\emph{Description and significance.} This result embeds a transfer
problem into a continual-learning stream by placing the source tasks before the
target task. Its significance is that continual learning formally contains the
source-target transfer setting when the continual objective includes the
target-risk requirement and the two interfaces share the stated transcript
schedule, learner closure, and raw-example cost.

{\color{black}
\paragraph{Worked case study: transporting a VC lower bound.}
Let \(\mathcal H\) be a binary hypothesis family with finite VC dimension
\(2\le d<\infty\),
use zero-one loss, and let \(\mathcal C_{\mathrm{real}}(\mathcal H)\) be the
class of distributions realizable by \(\mathcal H\). The standard realizable
PAC lower bound gives a universal constant \(c>0\) such that, for
\(\epsilon\in\mathscr I\cap(0,1/8]\) and \(0<\delta\le1/8\),
\[
N^\star_{\mathrm{Sup}}(
\mathcal C_{\mathrm{real}}(\mathcal H);\epsilon,\delta)
\ge
c\,\frac{d+\log(1/\delta)}{\epsilon}
\]
\citep{blumer1989learnability,shalev2014understanding,mohri2018foundations}.
Applying the two exact reductions yields, without a new lower-bound proof,
\begin{align*}
N^\star_{\mathrm{Trans}_K}(
\Phi_{\mathrm{Sup}\to\mathrm{Trans}_K}
(\mathcal C_{\mathrm{real}}(\mathcal H));\epsilon,\delta)
&\ge c\,\frac{d+\log(1/\delta)}{\epsilon},\\
N^\star_{\mathrm{Cont}_{K+1}}(
(\Phi_{\mathrm{Trans}_K\to\mathrm{Cont}_{K+1}}\circ
\Phi_{\mathrm{Sup}\to\mathrm{Trans}_K})
(\mathcal C_{\mathrm{real}}(\mathcal H));\tau_{\mathrm{Cont}},\delta)
&\ge c\,\frac{d+\log(1/\delta)}{\epsilon},
\end{align*}
for every continual threshold in the stated continual-learning accuracy regime
whose final-task coordinate is \(\epsilon\).
This concrete rate illustrates the framework's use: the hard supervised
subclass remains hard after it is embedded into richer transfer and continual
interfaces.

\paragraph{Literature-facing transfer schema.}
More specialized transfer lower bounds, for example those based on task
relatedness or diversity \citep{bendavid2010theory,kalan2020minimax}, can be
transported to continual learning in the same way, but only after their
environment class, protocol, hypothesis family, objective, and charged sample
unit have been matched to Assumption~\ref{ass:continual-target}. We therefore
present this as a reuse schema rather than claiming a new continual-learning
lower bound from unmatched models.
}

\begin{corollary}[Continual learning dominates supervised learning]
\label{cor:continual-dominates-supervised}
Under Assumptions~\ref{ass:degenerate-transfer} and~\ref{ass:continual-target},
\[
\mathrm{Cont}_{K+1} \succeq_{\mathrm{ex}} \mathrm{Sup}.
\]
\end{corollary}

\noindent\emph{Description and significance.} The corollary uses transitivity to
show that supervised learning is exactly contained in continual learning under
the two closure assumptions. It illustrates how the preorder structure lets
local reductions compose into broader relationships among paradigms.

The next result shows that meta-learning also contains supervised learning as a degenerate special case.

\begin{proposition}[Meta-learning dominates supervised learning]
\label{prop:meta-dominates-supervised}
Under Assumption~\ref{ass:singleton-meta},
\[
\mathrm{Meta} \succeq_{\mathrm{ex}} \mathrm{Sup}.
\]
\end{proposition}

\noindent\emph{Description and significance.} The proposition embeds supervised
learning into meta-learning by using a task distribution concentrated on one
task and causally reblocking charged examples within the same declared
budget into the support/query interface. It shows that
meta-learning's distinctiveness comes from non-degenerate task variation and
adaptation structure, not from the mere presence of an episodic protocol.

\begin{remark}[Sensitivity of the meta-learning reduction]
Proposition~\ref{prop:meta-dominates-supervised} relies on the ability to
partition supervised samples into the support/query structure with the
correct law, charge all revealed labels, and avoid declared-budget inflation.
If instead the meta setup charges completed
episodes, Proposition~\ref{prop:episode-example-accounting} gives the explicit
non-exact factor \(s+u\). This illustrates why the closure assumptions and
accounting convention are substantive rather than merely notational.
\end{remark}

\subsection{Resulting picture}

Collecting the previous propositions yields the following canonical hardness diagram:
\[
\mathrm{Cont}_{K+1} \succeq_{\mathrm{ex}} \mathrm{Trans}_K \succeq_{\mathrm{ex}} \mathrm{Sup},
\qquad
\mathrm{Meta} \succeq_{\mathrm{ex}} \mathrm{Sup}.
\]
Each arrow is conditional on the corresponding full-interface assumption
above; the diagram summarizes those canonical sanity checks rather than an
unconditional total order.

\begin{corollary}[Sample-complexity consequences for the canonical hierarchy]
\label{cor:canonical-sample-hierarchy}
Under Assumptions~\ref{ass:degenerate-transfer},
\ref{ass:continual-target}, and~\ref{ass:singleton-meta}, the exact reductions
above imply the following minimax inequalities on embedded comparison classes.

For every supervised comparison class \(\mathcal C_{\mathrm{Sup}}\), every
supervised threshold \(\tau_{\mathrm{Sup}}\in\mathscr I\), and every
\(\delta\in(0,1)\),
\[
N^\star_{\mathrm{Sup}}(\mathcal C_{\mathrm{Sup}};\tau_{\mathrm{Sup}},\delta)
\le
N^\star_{\mathrm{Trans}_K}(
\Phi_{\mathrm{Sup}\to\mathrm{Trans}_K}(\mathcal C_{\mathrm{Sup}});
\tau_{\mathrm{Sup}},\delta),
\]
\[
N^\star_{\mathrm{Sup}}(\mathcal C_{\mathrm{Sup}};\tau_{\mathrm{Sup}},\delta)
\le
N^\star_{\mathrm{Cont}_{K+1}}(
(\Phi_{\mathrm{Trans}_K\to\mathrm{Cont}_{K+1}}\circ
\Phi_{\mathrm{Sup}\to\mathrm{Trans}_K})(\mathcal C_{\mathrm{Sup}});
\tau_{\mathrm{Cont}},\delta),
\]
for every \(\tau_{\mathrm{Cont}}\) in the stated continual-learning accuracy
regime whose final-task projection equals \(\tau_{\mathrm{Sup}}\), and
\[
N^\star_{\mathrm{Sup}}(\mathcal C_{\mathrm{Sup}};\tau_{\mathrm{Sup}},\delta)
\le
N^\star_{\mathrm{Meta}}(
\Phi_{\mathrm{Sup}\to\mathrm{Meta}}(\mathcal C_{\mathrm{Sup}});
\tau_{\mathrm{Sup}},\delta).
\]
\end{corollary}

\noindent\emph{Description and significance.} This corollary translates the
canonical exact reductions into explicit minimax sample-complexity inequalities
for the embedded supervised, transfer, continual, and meta-learning classes. It
is significant because it turns the structural hierarchy into quantitative
consequences that can be used to transfer upper and lower bounds.

\begin{remark}[No universal total order]
The framework does not by itself imply a universal ordering between \(\mathrm{Meta}\) and \(\mathrm{Trans}\), or between \(\mathrm{Meta}\) and \(\mathrm{Cont}\). Such comparisons depend on additional structural assumptions about the environment classes and the objectives. For example, one would need a formal embedding of source-target problems into episodic task distributions, or conversely an embedding of episodic adaptation problems into sequential task streams. In this sense, the framework naturally produces a preorder rather than a total order.
\end{remark}

\begin{remark}[Interpretation of exact hardness]
An exact POR-reduction says that one paradigm contains another as a special case
without inflating the declared resource budget. This should be interpreted as a structural
statement about worst-case resource-bounded difficulty under aligned cost units. It does
not claim that optimization, constants, or empirical benchmark performance are
identical across paradigms.
\end{remark}

\begin{remark}[How existing lower bounds plug into the framework]
The point of the framework is not to replace within-paradigm lower-bound
analysis, but to make such results reusable. For instance, any lower bound
proved for a task-aware continual-learning model or for a large-replay variant
immediately transfers to the harder task-agnostic or smaller-memory variants by
Corollary~\ref{cor:lower-bound-transfer} together with
Propositions~\ref{prop:replay-monotonicity}
and~\ref{prop:taskid-monotonicity}. Likewise, lower bounds proved within
source-target transfer models \citep{kalan2020minimax} can be imported into
stronger structural settings once an appropriate POR embedding is specified.
The same principle applies to continual-learning variants whenever the relevant
replay, memory, or side-information choices are represented as POR resources.
In the present scalar theory, replay capacity is represented through the
admissible learner class and task identifiers through the protocol; only the
declared scalar transcript cost is measured by \(\kappa_P\).
\end{remark}

%% file: paper/related_work.tex
\section{Related Work}
\label{sec:related-work}

Classical learning theory relates sample complexity to a fixed hypothesis class
and loss through tools such as uniform convergence, stability, and capacity
measures
\citep{vapnik1971uniform,valiant1984theory,blumer1989learnability,vapnik1998statistical,anthony1999neural,shalev2014understanding,mohri2018foundations,bartlett2002rademacher,bousquet2002stability}.
Our focus is
orthogonal: the hypothesis family is fixed, while the protocol, objective, or
resources change. Transfer-learning theory studies when source tasks improve
target performance, including representation sharing, task relatedness, and
minimax lower bounds
\citep{caruana1997multitask,pan2010survey,bendavid2010theory,blitzer2008learning,crammer2008learning,mansour2009domain,weiss2016survey,maurer2016benefit,tripuraneni2020task,kalan2020minimax}.
Meta-learning and learning-to-learn analyze experience across related tasks and
episodic adaptation
\citep{thrun1998lifelong,baxter2000model,maurer2005stability,lake2015human,andrychowicz2016learning,santoro2016meta,vinyals2016matching,ravi2017optimization,snell2017prototypical,finn2017model,hospedales2022meta}.
Continual learning studies sequential exposure, stability-plasticity tradeoffs,
replay, task identifiers, and forgetting
\citep{robins1995catastrophic,french1999catastrophic,parisi2019continual,kirkpatrick2017overcoming,lopezpaz2017gradient,rebuffi2017icarl,pentina2015lifelong,delange2022continual,vanderven2022three}.
POR-reductions do not replace these within-paradigm analyses; they specify when
their guarantees or lower bounds transfer across paradigm interfaces.

The closest classical statistical analogue is the comparison of statistical
experiments. Blackwell's comparison theorem orders experiments by whether one
experiment can simulate another through randomization without increasing
decision risk \citep{blackwell1953equivalent}. Le Cam's theory of deficiency
and asymptotic comparison extends this viewpoint to approximate comparison of
statistical models \citep{lecam1964sufficiency,lecam1986asymptotic}, with
later systematic treatments such as \citep{torgersen1991comparison}. These
works compare information structures through decision performance, and they are
therefore conceptually close to the present paper. POR-reductions differ in the
objects being compared: rather than comparing statistical experiments for a
fixed decision problem, they compare learning paradigms that may differ in
environment class, observation protocol, admissible learner class, performance
functional, and resource accounting. The threshold map and budget overhead in a
POR-reduction make explicit two ingredients that are usually outside classical
experiment comparison: how learning objectives are translated and how sample
resources are calibrated across paradigms.
Moreover, POR comparisons are algorithm-interface comparisons: the compiler
must transform every admissible learner for the solver paradigm into a learner
for the reduced paradigm under the target protocol. The formal result is
complexity transfer in the declared accounting units (sample complexity under
labeled-example accounting); an efficient implementation requires an
additional computational guarantee on the compiler. This learner-class and
resource-sensitive perspective is what allows the framework to discuss method
transfer, replay memory, task identifiers, and adaptation protocols, which are
not naturally represented by an information ordering between experiments alone.

Reduction arguments are central in computational complexity
\citep{cook1971complexity,karp1972reducibility,garey1979computers,papadimitriou1994computational,arora2009computational}
and also appear within learning theory, for example in statistical-query
learning \citep{kearns1998statistical} and classification reductions
\citep{beygelzimer2005error}. The present work uses
reductions at the level of learning paradigms. The reduction data therefore
include an environment embedding, learner compiler, threshold map, and budget
calibration, which are precisely the ingredients needed for minimax
declared-resource-complexity transfer across protocol--objective--resource
specifications.

%% file: paper/conclusion.tex
\section{Conclusion}
\label{sec:conclusion}

We introduced a protocol--objective--resource framework for comparing learning
paradigms while holding the hypothesis family fixed. The central message is
that relationships among learning paradigms cannot be read from model classes
alone. They depend on how environments are embedded, how learners are compiled,
how success criteria are translated, and how samples and other resources are
accounted for. POR-reductions make these obligations explicit and convert them
into minimax comparisons in the declared accounting units; these are
sample-complexity comparisons when labeled examples are charged.
The revised formulation makes this statement operational: protocols induce
budgeted transcript laws, \(\kappa_P\) charges transcript cost, simulation
couplings expose protocol compatibility, and one learner at one common budget
must succeed over the full comparison class.

The framework gives a formal language for transferring understanding and
methods across learning regimes. On the understanding side, a valid reduction
transports upper bounds, lower bounds, and impossibility evidence between
embedded comparison classes. On the method side, the same reduction identifies
what must be preserved for an algorithmic idea to move from one paradigm to
another: the interaction transcript, the performance target, and the resource
budget. An efficient method-transfer theorem additionally needs a
constructive compiler with a computational bound. This helps distinguish
formally certified interface relationships from differences caused by side
information, objective strength, memory assumptions, or sample accounting. It
does not, by itself, decide empirical or scientific novelty. In this sense,
POR-reductions provide a way to organize the growing landscape of learning
regimes and to identify when an existing formal guarantee can be reused.

Our structural results support this view. The reduction theorem gives
declared-resource-complexity domination on embedded classes, its converse use yields
obstructions to proposed reductions, and the preorder structure allows local
comparisons to be composed into larger diagrams. The objective-strengthening
examples show that the objective component is not cosmetic: even with the same
protocol and learner class, changing what must be controlled can strictly
increase minimax sample complexity. The instantiations on supervised, transfer,
continual, and meta-learning abstractions demonstrate how familiar learning
regimes can be related through exact special-case containments under explicit
closure assumptions, and are positioned as canonical sanity checks of the
formal obligations. The episode-to-example construction supplies a genuinely
non-exact resource conversion, while the VC case study instantiates lower-bound
transport with an explicit rate. The replay-memory and task-identifier examples show
that the same language also compares variants inside a single paradigm.

\subsection{Limitations and Future Work}

The framework is intentionally structural and worst-case. It does not claim
that a POR reduction preserves optimization difficulty, constants, empirical
benchmark performance, or typical-case behavior. It also does not impose a
universal total order over learning paradigms. Comparisons depend on the chosen
environment classes, accuracy regimes, closure assumptions, and resource
conventions; changing any of these ingredients can change the resulting
hardness relation.

Several extensions are natural. First, the resource model should be enriched
beyond a single scalar charged resource. Many practical comparisons also depend on memory,
computation, communication, adaptation rounds, replay buffers, and access to
task identifiers. A multi-resource POR theory could make these tradeoffs
explicit. Second, approximate reductions could allow controlled slack in the
performance threshold as well as in the resource budget, better matching
empirical method transfer where small losses are often accepted. Third,
distribution-dependent or randomized reductions could capture settings in which
structural transfer holds for most tasks under a meta-distribution rather than
uniformly over a comparison class. Finally, the framework should be paired with
sharper within-paradigm sample-complexity results. POR-reductions are most
useful when concrete upper and lower bounds are available to transport across
learning interfaces.

The abstract POR object can also describe active and semi-supervised
learning, online or bandit feedback, federated interfaces, and reinforcement
learning once their transcript laws, objectives, and charged units are stated.
Representability alone does not establish a reduction: each application must
still supply an environment embedding, an admissible causal compiler, objective
transfer, and resource calibration. Settings that require simultaneous sample,
communication, and computation budgets motivate the multi-resource extension
rather than an immediate application of the present scalar theorem.

%% file: paper/broader_impact.tex
\section*{Broader Impact Statement}

This paper is theoretical and does not introduce a deployable learning system,
dataset, or benchmark. Its main purpose is to clarify when sample-complexity
claims can be transferred across learning paradigms. The most direct positive
impact is methodological: making protocol, objective, and resource assumptions
explicit can reduce misleading comparisons between learning settings. Because the work is conceptual rather than application-oriented, we do not anticipate significant immediate societal risks or direct downstream impacts beyond its potential influence on future machine learning research.

%% file: paper/appendix_discussion.tex
\section{Further Discussion of the Framework}
\label{app:framework-discussion}

This appendix expands on how the protocol--objective--resource framework should
be interpreted and used. The main text states the formal definitions and main
sample-complexity consequences; Appendix~\ref{app:detailed-proofs} gives the
detailed proofs. The purpose here is more conceptual: to clarify what a
reduction certifies, what kinds of transfer it supports, and where additional
assumptions are still needed.

\subsection{What a POR reduction certifies}
\label{app:discussion-certifies}

A POR reduction is best read as a structural simulation certificate. It says
that solving one learning interface is sufficient for solving another, provided
that four pieces of reduction data are supplied: environments must be embedded,
learners must be compiled, success criteria must be translated, and resource
budgets must be calibrated. Their induced coupling must then pass the three
checks of admissibility and protocol validity, resource calibration, and
objective transfer. Each component rules out a common source of informal ambiguity.

The environment embedding specifies which instances of the reduced paradigm are
being represented inside the solver paradigm. The learner compiler specifies
how an algorithm designed for the solver paradigm is actually run in the
reduced paradigm. The threshold map specifies why success for the solver
objective implies success for the reduced objective. The resource overhead
specifies the cost of the simulation. A proposed comparison that omits one of
these pieces may still be intuitively useful, but it is not yet a POR
reduction.

This interpretation is deliberately algorithm-independent. The compiler acts
on every admissible learner for the solver paradigm, not just on a particular
algorithm. This is what makes the relation a hardness relation rather than a
statement about one method. If every solver for paradigm \(A\) can be compiled
into a solver for paradigm \(B\), then the difficulty of \(B\) is already
present inside \(A\) under the stated comparison class and accuracy regime.

\subsection{Understanding transfer and method transfer}
\label{app:discussion-method-transfer}

The framework separates two notions that are often conflated in practice.
First, there is \emph{understanding transfer}: structural facts about one
paradigm, such as upper bounds, lower bounds, monotonicity, or impossibility
results, may imply corresponding facts about another paradigm. This is the
role of the declared-resource-complexity transfer theorem and its lower-bound
consequence. Once a reduction is established, a theorem proved in one setting
can be interpreted in the other without reproving the entire result from
scratch.

Second, there is \emph{method transfer}: an actual learning procedure designed
for one paradigm may be reused in another. A POR reduction makes explicit what
has to be preserved for this reuse to be valid. The target interaction must
provide the transcript required by the method, the success criterion in the
source setting must imply the desired criterion in the target setting, and all
additional costs must be counted. This is why shared architectures or
optimizers do not by themselves establish method transfer. The protocol,
objective, and resources must also align.

Strictly speaking, the POR theorem proves an interface-level complexity
consequence in the declared accounting units. Calling the compiler a practical method-transfer
procedure additionally requires it to be constructive and to satisfy whatever
computational, communication, or memory overhead is relevant. Without those
extra conditions, POR supplies an auditable checklist rather than an efficiency guarantee.

This distinction is useful when diagnosing empirical claims. If a transfer
method fails, the failure may come from representation, but it may also come
from a broken protocol simulation, a mismatch between training and evaluation
objectives, or an unaccounted resource such as replay memory, adaptation
queries, or task identifiers. The POR language gives a checklist for locating
which part of the claimed transfer is unsupported.

\subsection{The role of comparison classes}
\label{app:discussion-comparison-classes}

Comparison classes are essential because hardness is not a property of a
paradigm name alone. Transfer learning with unrelated source and target tasks
is different from transfer learning with a shared representation. Continual
learning with adversarial task switches is different from continual learning
with slowly drifting tasks. Meta-learning over a broad family of unrelated
tasks is different from meta-learning over a tightly structured task family.

The same named paradigm can therefore occupy different places in the hardness
preorder depending on the comparison class. This is a feature rather than a
defect. It prevents the framework from making overly broad statements such as
``transfer learning is always easier'' or ``continual learning is always
harder.'' Instead, a comparison must state which environments are being
compared and how they are embedded.

In applications, the comparison class is where domain assumptions enter. A
shared representation assumption, a task-relatedness assumption, a smooth drift
assumption, or a bounded-memory assumption should be encoded in the environment
class or learner class before a reduction is claimed. This makes the
statistical content of the comparison visible.

\subsection{Accuracy regimes and meaningful objectives}
\label{app:discussion-accuracy-regimes}

Accuracy regimes state the intended thresholds and help diagnose formally correct but uninformative reductions. Without
restricting the thresholds under consideration, one could map every solver
requirement to a threshold that is automatically satisfied with no data. Such a
map would satisfy a formal implication but would not express a meaningful
hardness relation.

\begin{proposition}[Vacuity from loose thresholds]
\label{prop:vacuity-without-regimes}
Suppose the accuracy regime \(\mathscr T_B\) for \(B\) contains a threshold
\(\bar\tau_B\) for which there exists a learner
\(T_B^0\in\mathfrak T_B\) satisfying
\[
\Pr[\mathbf J_B(T_B^0,e;0)\preceq \bar\tau_B]=1
\qquad\text{for all }e\in\mathcal E_B.
\]
If \(\mathcal E_A\) and \(\mathscr T_A\) are nonempty, then
\(B\preceq_{\mathrm{ex}} A\) relative to
\((\mathscr T_A,\mathscr T_B)\) for every paradigm \(A\).
\end{proposition}

\noindent\emph{Description and significance.} This proposition shows that if a
permitted threshold can be met with zero charged units, then \(B\) exactly reduces to
every nonempty paradigm \(A\). The result explains why accuracy regimes are
necessary but not sufficient: the regime and threshold map must exclude
vacuous thresholds, or the framework would certify
comparisons that carry no substantive learning content.

\noindent\emph{Proof deferred to Appendix~\ref{app:proof-vacuity-without-regimes}.}

The regime should therefore be chosen to reflect the scientific question. For
scalar risk, a regime might contain only thresholds below a baseline. For
continual learning, a regime might require simultaneous control of final-task
risk and forgetting. For meta-learning, a regime might focus on
post-adaptation query risk after a fixed support budget. Different regimes can
lead to different valid reductions because they ask different questions about
success.

This also explains the objective-strengthening examples in the main text. If
one paradigm requires control of only one performance coordinate and another
requires control of all coordinates, then the second objective can be strictly
harder even when the protocol and learner class are unchanged. The difference
comes from what success means, not from what data are observed.

\subsection{Resource accounting}
\label{app:discussion-resources}

The resource component is the mechanism that keeps reductions honest. A
simulation may require additional labeled samples, source tasks, support
examples, adaptation queries, memory, replay buffers, communication, or
computation. The present formalism tracks one declared scalar cost; the
canonical objects use labeled-example cost, while the non-exact example uses
episode cost. Other restrictions may be encoded through the admissible learner
class or protocol, but a reduction must state what is being spent.

Formally, \(\kappa_P\) now assigns cost to the full execution record, and an
\(M\)-budget protocol stops before that cost exceeds \(M\). The canonical
objects charge one unit per revealed labeled example. Proposition
\ref{prop:episode-example-accounting} changes the accounting unit to completed
episodes and gives an explicit conversion with multiplicative overhead
\(s+u\); its two-environment restriction proves failure of exactness, not
minimality of that affine pair. Replay capacity remains a restriction on
\(\mathfrak T_P\), while task-identifier access remains part of
\(\mathcal O_P\); neither is silently folded into the scalar cost.

This point is especially important for method transfer. For example, a
continual-learning method that uses a large replay buffer may look comparable
to a memoryless method if only labeled examples are counted. A meta-learning
method may appear sample-efficient if support and query examples are charged
under different conventions. A transfer-learning method may appear to improve
target performance while relying on a large amount of source data. POR
reductions force these choices into the statement of the comparison.

Future versions of the framework could make the resource vector more explicit,
tracking samples, memory, computation, and adaptation rounds simultaneously.
The affine-overhead form used here is a first step: it captures one declared
scalar transcript cost while keeping the main definitions simple.

\subsection{Refinements within a fixed paradigm}
\label{app:discussion-fixed-paradigm}

The framework is equally useful for comparing \emph{variants} of the same named
paradigm. Two common axes are side information and memory budget.

\begin{proposition}[Replay-budget monotonicity in continual learning]
\label{prop:replay-monotonicity}
Fix a continual-learning environment class and performance functional. For each
replay budget \(r \in \mathbb N\), let \(\mathrm{Cont}(r)\) denote the
continual-learning paradigm whose admissible learners may store at most \(r\)
examples in replay memory. If \(r \le r'\), then
\[
\mathrm{Cont}(r) \succeq_{\mathrm{ex}} \mathrm{Cont}(r').
\]
\end{proposition}

\noindent\emph{Description and significance.} The result states that a
continual learner with a larger replay budget can simulate any learner with a
smaller replay budget. It is significant because replay memory becomes an
ordered resource in the POR framework, allowing variants of continual learning
to be compared by their admissible learner classes.

\noindent\emph{Proof deferred to Appendix~\ref{app:proof-replay-monotonicity}.}

\begin{proposition}[Task-identifier side information in continual learning]
\label{prop:taskid-monotonicity}
Fix a continual-learning environment class, learner class, performance
functional, and sample accounting rule. Let
\(\mathrm{Cont}^{\mathrm{id}}\) denote the version of the protocol that reveals
task identifiers with each observed example, and let
\(\mathrm{Cont}^{\mathrm{blind}}\) denote the version that does not reveal these
identifiers. Then
\[
\mathrm{Cont}^{\mathrm{blind}} \succeq_{\mathrm{ex}} \mathrm{Cont}^{\mathrm{id}}.
\]
\end{proposition}

\noindent\emph{Description and significance.} This proposition formalizes that
task-identifier side information cannot make the continual-learning problem
harder, since a task-aware learner may ignore the identifiers. The result
clarifies how protocol-level side information induces a hardness ordering even
inside a single named paradigm.

\noindent\emph{Proof deferred to Appendix~\ref{app:proof-taskid-monotonicity}.}

\begin{remark}[Design implications]
Propositions~\ref{prop:replay-monotonicity}
and~\ref{prop:taskid-monotonicity} illustrate a point that is easy to miss if
one speaks only at the level of broad paradigm labels: many practically
important design choices themselves induce hardness orderings. The POR language
therefore provides a way to compare not only supervised versus transfer versus
meta-learning, but also task-aware versus task-agnostic continual learning, or
small-memory versus large-memory variants of the same setup.
\end{remark}

\begin{corollary}[Quantitative monotonicity for continual-learning variants]
\label{cor:continual-variant-monotonicity}
Under the assumptions of
Propositions~\ref{prop:replay-monotonicity}
and~\ref{prop:taskid-monotonicity}, the corresponding minimax sample
complexities satisfy the monotone inequalities
\[
N_{\mathrm{Cont}(r')}^\star(\mathcal C;\tau,\delta)
\le
N_{\mathrm{Cont}(r)}^\star(\mathcal C;\tau,\delta)
\qquad\text{whenever } r \le r',
\]
and
\[
N_{(\mathrm{Cont}^{\mathrm{id}})}^\star(\mathcal C;\tau,\delta)
\le
N_{(\mathrm{Cont}^{\mathrm{blind}})}^\star(\mathcal C;\tau,\delta).
\]
\end{corollary}

\noindent\emph{Description and significance.} The corollary converts the replay
and task-identifier monotonicity reductions into explicit minimax
sample-complexity inequalities. Its significance is that practical design
choices such as memory and side information lead to quantitative consequences,
not only qualitative intuitions.

\noindent\emph{Proof deferred to Appendix~\ref{app:proof-continual-variant-monotonicity}.}

\subsection{What the framework does not claim}
\label{app:discussion-nonclaims}

The framework gives worst-case structural statements. It does not claim that
two paradigms have the same optimization difficulty, empirical performance, or
typical-case behavior. It also does not say that a harder paradigm is always
less useful in practice. A harder paradigm may provide more information, more
flexible objectives, or better inductive structure for a particular
application.

Nor does the framework impose a universal total order over supervised,
transfer, continual, and meta-learning. Some comparisons require additional
assumptions, and some may be incomparable under natural regimes. This is why
the formal relation is a preorder. The goal is to identify valid structural
relations where they exist, and to explain precisely which obligation fails
when they do not.

\subsection{Possible extensions}
\label{app:discussion-extensions}

Several extensions are natural. One direction is to develop approximate
reductions that allow controlled slack in the performance threshold, not only
in the declared-resource budget. This would better match empirical practice, where one
often accepts small losses in accuracy when moving a method across settings.

A second direction is to study randomized or distribution-dependent reductions.
The current definition is uniform over the comparison class. In some
applications, one may want reductions that hold for most environments under a
meta-distribution, or reductions that adapt to observable structure in the
task family.

A third direction is to combine POR reductions with sharper within-paradigm
lower bounds. The framework is most useful when paired with concrete
statistical results: once a lower bound is known for one paradigm, a reduction
can move that lower bound to another. Conversely, a separation in sample
complexity can rule out a proposed structural comparison. In this sense, the
framework is not a replacement for problem-specific learning theory, but a
language for transporting it across paradigms.

Beyond the four representative paradigms, the same object can encode
active learning (label-query transcripts), semi-supervised learning (labeled
and unlabeled access), online or bandit learning (adaptive feedback and regret),
federated learning (client sampling and communication constraints), and
reinforcement learning (trajectory protocols). These are scope examples, not
claimed reductions. Each requires its own protocol law, objective, admissible
learner class, accounting unit, and reduction witnesses. When two or more costs
must be controlled simultaneously, the appropriate next step is the proposed
multi-resource extension rather than collapsing them without justification.

%% file: paper/appendix_proofs.tex
\section{Detailed Proofs}
\label{app:detailed-proofs}

This appendix collects the detailed proofs for results stated in the main text
and in Appendix~\ref{app:framework-discussion}. The proofs use the notation and
standing assumptions introduced in Section~\ref{sec:preliminaries-framework}.

\subsection{Proof of Proposition~\ref{prop:vacuity-without-regimes}}
\label{app:proof-vacuity-without-regimes}

\begin{proof}
{\color{black}
We verify the POR-reduction definition directly. Since \(\mathcal E_A\) is
nonempty, choose an arbitrary reference environment
\(e_A^0\in\mathcal E_A\), and define the embedding
\[
\Phi(e):=e_A^0
\qquad\text{for every }e\in\mathcal E_B .
\]
Since \(\mathscr T_A\) is nonempty and
\(\bar\tau_B\in\mathscr T_B\), define the threshold map
\[
\Gamma(\tau_A):=\bar\tau_B
\qquad\text{for every }\tau_A\in\mathscr T_A .
\]

The only point requiring care is the budget in the conclusion of the
reduction. The learner \(T_B^0\) is known to satisfy the \(B\)-threshold at
zero charged units. To use the same zero-resource behavior under an arbitrary external
budget \(M\), let \(q_0(M)=0\) for all \(M\). By the budget-capping closure in
Assumption~\ref{ass:budget-capping}, the capped learner
\((T_B^0)^{q_0}\) is admissible for \(B\), and
\[
    \mathbf J_B((T_B^0)^{q_0},e;M)
    =
    \mathbf J_B(T_B^0,e;0)
    \]
    almost surely under the cap-wrapper coupling, for every \(e\) and every
    \(M\). Define the compiler by
\[
\mathfrak C(T_A):=(T_B^0)^{q_0}
\qquad\text{for every }T_A\in\mathfrak T_A .
\]
For each \(T_A\) and \(M\), the compiler's causal simulator ignores the
\(B\)-history and uses fresh private randomness to generate the operational
\(A\)-execution law of \(T_A\) in the fixed reference environment \(e_A^0\).
This sequential kernel has no argument \(e\), has the required \(A\)-marginal,
and is permitted by the extensional definition; no claim of efficient
sampling is made. Couple it with the capped \(B\)-execution above. The two
marginals are therefore protocol-valid, and their costs are at most \(M\).

Now fix arbitrary \(e\in\mathcal E_B\), \(T_A\in\mathfrak T_A\),
\(\tau_A\in\mathscr T_A\), \(\delta\in(0,1)\), and \(M\in\mathbb N\).
Regardless of whether the antecedent success event for \(A\) holds, the
compiled learner satisfies
\[
\Pr\!\left[
\mathbf J_B(\mathfrak C(T_A),e;M)\preceq\Gamma(\tau_A)
\right]
=
\Pr\!\left[
\mathbf J_B(T_B^0,e;0)\preceq\bar\tau_B
\right]
=1 .
\]
The \(B\)-success event thus has probability one, so it contains the
\(A\)-success event up to a null event under this simulator-induced coupling.
Thus the POR conditions hold with overhead \(a=1\), \(b=0\). This is an
exact reduction. Because the construction ignores \(A\) entirely, the
reduction carries no meaningful hardness information; it is precisely the
vacuity that accuracy regimes are meant to exclude.
}
\end{proof}

\subsection{Proof of Proposition~\ref{prop:budget-calibration}}
\label{app:proof-budget-calibration}

\begin{proof}
{\color{black}
The assumed causal simulator runs the pre-compiled learner
\(\mathfrak C_0(T_A)\) with internal budget \(\psi(M)\). The POR definition,
however, supplies the compiled \(B\)-learner with the external budget \(aM+b\).
We therefore cap the compiled learner so that, whenever the external budget
equals \(aM+b\), its internal budget is \(\psi(M)\).

For an external budget \(M'\in\mathbb N\), define
\[
m(M'):=\max\{m\in\mathbb N:am+b\le M'\},
\]
with \(m(M')=0\) if the displayed set is empty. Define
\[
q(M') :=
\begin{cases}
\psi(m(M')),& M'\ge b,\\
0,& M'<b .
\end{cases}
\]
We first check that \(q\) is a valid cap. If \(M'<b\), then
\(q(M')=0\le M'\). If \(M'\ge b\), then \(am(M')+b\le M'\) by the definition
of \(m(M')\). Since \(\psi(m)\le am+b\) for every \(m\), we get
\[
q(M')=\psi(m(M'))\le am(M')+b\le M' .
\]
Thus \(q(M')\le M'\) for all external budgets \(M'\). By
Assumption~\ref{ass:budget-capping}, the learner
\[
\mathfrak C(T_A):=(\mathfrak C_0(T_A))^q
\]
is admissible for \(B\), and, for every \(e\), its canonical cap-wrapper
coupling satisfies
\[
Z_{B,\mathfrak C(T_A),e}^{M'}
=
Z_{B,\mathfrak C_0(T_A),e}^{q(M')},
\qquad
\mathbf J_B(\mathfrak C(T_A),e;M')
=
\mathbf J_B(\mathfrak C_0(T_A),e;q(M'))
\quad\text{almost surely}.
\]

Now set \(M'=aM+b\). Since \(a\ge 1\), the condition
\[
am'+b\le aM+b
\]
is equivalent to \(m'\le M\), and hence \(m(aM+b)=M\). Therefore
\[
q(aM+b)=\psi(M).
\]

Fix \(e,T_A,M\), and realize the capped execution at external budget
\(aM+b\) by running the pre-compiler's simulator with the same protocol
responses and private randomness as its internal \(\psi(M)\)-budget
execution. Denote its transcript and performance by \(\widehat Z_B\) and
\(\widehat{\mathbf J}_B\), respectively. On this simulator-induced joint
space, the cap wrapper does not alter the internal execution, so
\[
\widehat{\mathbf J}_B
=\mathbf J_B(\mathfrak C_0(T_A),e;\psi(M))
\quad\text{almost surely}.
\]
The hatted variable is an actual execution of \(\mathfrak C(T_A)\) at external
budget \(aM+b\); by the cap-wrapper coupling,
\[
\mathbf J_B(\mathfrak C(T_A),e;aM+b)
=
\mathbf J_B(\mathfrak C_0(T_A),e;\psi(M))
\quad\text{almost surely}.
\]
Thus the pre-compiler's assumed pathwise inclusion yields, for every
\(\tau_A\in\mathscr T_A\),
\[
\{\mathbf J_A(T_A,\Phi(e);M)\preceq\tau_A\}
\subseteq
\{\widehat{\mathbf J}_B\preceq\Gamma(\tau_A)\}
\]
up to the same null event.

The cap wrapper is deterministic and environment-independent, so composing it
with \(\mathsf S^{0,A\leftarrow B}_{T_A,M}\) preserves causality and
environment independence. Its two execution marginals are protocol-valid, and
\[
\kappa_B(\widehat Z_B)\le aM+b
\quad\text{almost surely}.
\]
The \(A\)-cost bound is inherited from its operational \(M\)-budget execution.
Hence admissibility, protocol validity, objective transfer, and resource
calibration all hold with overhead parameters \((a,b)\).
}
\end{proof}

{\color{black}
\subsection{Proof of Proposition~\ref{prop:episode-example-accounting}}
\label{app:proof-episode-accounting}

\begin{proof}
Use the identity environment embedding and identity threshold map. Given an
episode-accounted learner \(T\), let \(\mathfrak C(T)\) run the same causal
policy on the same completed episodes under example accounting. At episode
budget \(M\), \(P_{\mathrm{epi}}^{s,u}\) exposes \(M\) complete episodes.
Acquiring the identical raw transcript under example accounting costs exactly
\(rM=(s+u)M\), so the identity coupling gives
\[
Z_{P_{\mathrm{ex}},\mathfrak C(T),e}^{rM}
=
Z_{P_{\mathrm{epi}},T,e}^{M}
\]
and therefore equality of the performance variables. Canonical closure gives
learner admissibility, and
\(\kappa_{\mathrm{ex}}=r\kappa_{\mathrm{epi}}\) gives resource calibration.
Thus the POR conditions hold with overhead \((r,0)\).

It remains to show that exactness can fail. Let
\(\mathcal X=\{x_0\}\), \(\mathcal Y=\{0,1\}\), and
\(\mathcal H=\{h_0,h_1\}\), where \(h_\theta(x_0)=\theta\). For
\(\theta\in\{0,1\}\), let \(D_\theta\) put unit mass on
\((x_0,\theta)\), and take meta-environments
\(\Pi_\theta=\delta_{D_\theta}\). Let
\(\mathcal C:=\{\Pi_0,\Pi_1\}\), equip both two-environment restrictions with
the population zero-one-risk functional and singleton threshold regime
\(\{0\}\), and let the common learner class contain all randomized causal
policies from completed-episode transcripts to \(\mathcal H\) (hence the
reveal-and-output policy used below). Fix any
\(\delta\in(0,1/2)\). Before one complete episode is revealed, the transcript
is identical for \(\theta=0\) and \(\theta=1\), so no single learner can
succeed at threshold \(0\) on the full class \(\mathcal C\) with probability
greater than \(1/2\). Because \(s\ge1\), one complete episode reveals
\(\theta\) through a learner-visible support label and permits zero risk
uniformly on \(\mathcal C\). Hence
\[
N_{P_{\mathrm{epi}}^{s,u}}^\star
(\mathcal C;0,\delta)=1.
\]
Under example accounting no completed episode is exposed before budget \(r\),
whereas one is exposed at budget \(r\). Therefore
\[
N_{P_{\mathrm{ex}}^{s,u}}^\star
(\mathcal C;0,\delta)=r.
\]
Suppose, for contradiction, that there were an exact POR-reduction from the
example-accounted restriction on \(\mathcal C\) to the episode-accounted
restriction on \(\mathcal C\). Its environment embedding
\(\Phi:\mathcal C\to\mathcal C\) may be arbitrary. Because both accuracy
regimes are the singleton \(\{0\}\), its threshold map must satisfy
\(\Gamma(0)=0\). Exact declared-resource-complexity domination and monotonicity in
comparison classes would then give
\[
r
=N_{P_{\mathrm{ex}}^{s,u}}^\star(\mathcal C;0,\delta)
\le
N_{P_{\mathrm{epi}}^{s,u}}^\star(\Phi(\mathcal C);0,\delta)
\le
N_{P_{\mathrm{epi}}^{s,u}}^\star(\mathcal C;0,\delta)
=1,
\]
where the second inequality uses
\(\Phi(\mathcal C)\subseteq\mathcal C\). This contradicts \(r\ge2\).
Thus no exact reduction exists in the displayed direction between the two
restricted paradigms. The argument establishes the need for some nonzero
overhead; it does not establish that \((r,0)\) is the minimal affine pair.
\end{proof}
}

\subsection{Proof of Theorem~\ref{thm:reduction-domination}}
\label{app:proof-reduction-domination}

\begin{proof}
{\color{black}
Write
\[
n_A:=N_A^\star(\Phi(\mathcal C_B);\tau_A,\delta).
\]
If \(n_A=\infty\), the claimed inequality is immediate. Suppose
\(n_A<\infty\). Since feasible budgets form a nonempty subset of
\(\mathbb N\), the infimum in~\eqref{eq:uniform-minimax} is attained. Hence
there exists one learner \(T_A\in\mathfrak T_A\) such that, simultaneously for
every \(e'\in\Phi(\mathcal C_B)\),
\[
\Pr[\mathbf J_A(T_A,e';n_A)\preceq\tau_A]\ge1-\delta.
\]

Let \(T_B:=\mathfrak C(T_A)\). For any \(e\in\mathcal C_B\), objective
transfer under the simulation coupling gives
\[
\Pr[\mathbf J_B(T_B,e;an_A+b)\preceq\Gamma(\tau_A)]
\ge
\Pr[\mathbf J_A(T_A,\Phi(e);n_A)\preceq\tau_A]
\ge1-\delta.
\]
Admissibility and resource calibration ensure that the same \(T_B\) is valid
and has charged cost at most \(an_A+b\) for every \(e\). Thus one learner and
one common budget work over \(\mathcal C_B\), so
\[
N_B^\star(\mathcal C_B;\Gamma(\tau_A),\delta)
\le an_A+b.
\]
This is the desired inequality.
}
\end{proof}

\subsection{Proof of Corollary~\ref{cor:lower-bound-transfer}}
\label{app:proof-lower-bound-transfer}

\begin{proof}
Theorem~\ref{thm:reduction-domination} gives
\[
N_B^\star(\mathcal C_B;\Gamma(\tau_A),\delta)
\le
a\,N_A^\star(\Phi(\mathcal C_B);\tau_A,\delta)+b .
\]
If \(N_A^\star(\Phi(\mathcal C_B);\tau_A,\delta)=\infty\), the claimed lower
bound is automatic. Otherwise, the right-hand side in the theorem is finite,
so the theorem also implies that \(N_B^\star\) is finite. We may then subtract
\(b\) and divide by the positive integer \(a\), obtaining
\[
N_A^\star(\Phi(\mathcal C_B);\tau_A,\delta)
\ge
\frac{N_B^\star(\mathcal C_B;\Gamma(\tau_A),\delta)-b}{a}.
\]
Thus any lower bound on the reduced-paradigm quantity for \(B\) transfers to the
embedded quantity for \(A\), with the infinite case covered by the
extended-real convention.
\end{proof}

\subsection{Proof of Corollary~\ref{cor:exact-domination}}
\label{app:proof-exact-domination}

\begin{proof}
Exactness means that the POR-reduction has overhead parameters \(a=1\) and
\(b=0\). Substituting these values into
Theorem~\ref{thm:reduction-domination} gives
\[
N_B^\star(\mathcal C_B;\Gamma(\tau_A),\delta)
\le
N_A^\star(\Phi(\mathcal C_B);\tau_A,\delta),
\]
which is the desired statement.
\end{proof}

\subsection{Proof of Corollary~\ref{cor:reduction-obstruction}}
\label{app:proof-reduction-obstruction}

\begin{proof}
Assume, for contradiction, that a POR-reduction exists with the candidate data
\((\Phi,\Gamma,a,b)\). Then Theorem~\ref{thm:reduction-domination} must hold
for every comparison class, threshold, and confidence level. In particular, it
would imply
\[
N_B^\star(\mathcal C_B;\Gamma(\tau_A),\delta)
\le
a\,N_A^\star(\Phi(\mathcal C_B);\tau_A,\delta)+b
\]
for the class, threshold, and confidence level appearing in the corollary.
This contradicts the displayed strict inequality in the statement. Therefore
no reduction with those candidate data exists. If every possible embedding and
threshold map admits such a witness, the same contradiction rules out every
POR-reduction with the specified overhead parameters.
\end{proof}

\subsection{Proof of Proposition~\ref{prop:threshold-monotonicity}}
\label{app:proof-threshold-monotonicity}

\begin{proof}
{\color{black}
Because \(\tau_P\preceq\tau_P'\), every outcome that satisfies the stricter
threshold \(\tau_P\) also satisfies the weaker threshold \(\tau_P'\):
\[
\{\mathbf J_P(T,e;M)\preceq\tau_P\}
\subseteq
\{\mathbf J_P(T,e;M)\preceq\tau_P'\}.
\]
Consequently, any learner and common budget that succeeds uniformly on
\(\mathcal C_P\) at threshold \(\tau_P\) also succeeds uniformly at
\(\tau_P'\). Taking the smallest feasible uniform budgets gives
\[
N_P^\star(\mathcal C_P;\tau_P,\delta)
\ge
N_P^\star(\mathcal C_P;\tau_P',\delta).
\]
}
\end{proof}

\subsection{Proof of Proposition~\ref{prop:class-monotonicity}}
\label{app:proof-class-monotonicity}

\begin{proof}
{\color{black}
If one learner at one budget succeeds on every environment in
\(\mathcal C_P'\), it succeeds at the same budget on the subclass
\(\mathcal C_P\subseteq\mathcal C_P'\). Therefore the feasible
learner--budget pairs for the larger class form a subset of those for the
smaller class, and
\[
N_P^\star(\mathcal C_P;\tau_P,\delta)
\le
N_P^\star(\mathcal C_P';\tau_P,\delta).
\]
}
\end{proof}

\subsection{Proof of Proposition~\ref{prop:preorder}}
\label{app:proof-preorder}

\begin{proof}
{\color{black}
It is enough to verify reflexivity and transitivity for the reduction relation
\(\preceq_{\mathrm{POR}}\). For reflexivity, use the identity environment map,
compiler, threshold map, and identity causal simulation kernels, with overhead
\((1,0)\). All protocol, cost, and objective obligations are identities, so
\(A\preceq_{\mathrm{ex}}A\).

For transitivity, suppose
\[
C \preceq_{\mathrm{POR}} B
\quad\text{and}\quad
B \preceq_{\mathrm{POR}} A .
\]
Write the two reductions as
\[
(\Phi_{CB},\mathfrak C_{CB},\Gamma_{CB},a_{CB},b_{CB})
\]
for the reduction from \(C\) to \(B\), and
\[
(\Phi_{BA},\mathfrak C_{BA},\Gamma_{BA},a_{BA},b_{BA})
\]
for the reduction from \(B\) to \(A\). Define the composed data by
\[
\Phi_{CA}:=\Phi_{BA}\circ\Phi_{CB},\qquad
\mathfrak C_{CA}:=\mathfrak C_{CB}\circ\mathfrak C_{BA},\qquad
\Gamma_{CA}:=\Gamma_{CB}\circ\Gamma_{BA}.
\]
The proposed overhead is
\[
a_{CA}:=a_{CB}a_{BA},
\qquad
b_{CA}:=a_{CB}b_{BA}+b_{CB}.
\]

Fix \(T_A\in\mathfrak T_A\), \(M\in\mathbb N\), and set
\[
T_B:=\mathfrak C_{BA}(T_A),
\qquad
M_B:=a_{BA}M+b_{BA},
\qquad
M_C:=a_{CB}M_B+b_{CB}.
\]
The compiler composition is admissible because each compiler maps into the
domain of the next. The two environment-independent causal simulators compose
sequentially as
\[
\mathsf S^{A\leftarrow C}_{T_A,M}
:=
\mathsf S^{A\leftarrow B}_{T_A,M}
\circ_{\mathrm{seq}}
\mathsf S^{B\leftarrow C}_{T_B,M_B}.
\]
Concretely, the right-hand simulator uses the revealed \(C\)-history to emit a
virtual \(B\)-history, and the left-hand simulator uses prefixes of that
virtual history to emit the virtual \(A\)-history. By definition, both component
simulators use adapted, nondecreasing, almost-surely finite stopping schedules.
By the kernel definition, the virtual \(B\)-operational filtration through step
\(s\), including the virtual learner/compiler randomness drawn by then, is
measurable with respect to the \(C\)-operational filtration stopped at the
inner simulator's time \(\sigma_s^{B\leftarrow C}\). The standard time-change
property for filtrations therefore implies that a stopping time for the virtual
\(B\)-filtration, evaluated along this nondecreasing schedule, is an
almost-surely finite stopping time for the \(C\)-filtration; the composed
schedule remains adapted and nondecreasing.
Hence the composed kernel is causal. It has no environment argument because
neither component kernel has one.

Now fix \(e\in\mathcal E_C\) and run this kernel with the \(C\)-protocol. The
first simulation has the operational \(C\)-marginal and produces the
operational \(B\)-execution in environment \(\Phi_{CB}(e)\), with learner
\(T_B\) and budget \(M_B\). Conditional on that virtual execution, the second
simulation produces the operational \(A\)-execution in environment
\(\Phi_{BA}(\Phi_{CB}(e))=\Phi_{CA}(e)\), with learner \(T_A\) and budget
\(M\). Thus all three executions occur on the single joint law generated by
the sequential kernel; no environment-specific coupling is introduced. Their
costs satisfy
\[
\kappa_A(Z_A)\le M,
\qquad
\kappa_B(Z_B)\le M_B,
\qquad
\kappa_C(Z_C)\le M_C
\quad\text{almost surely},
\]
where
\[
M_C=a_{CB}(a_{BA}M+b_{BA})+b_{CB}=a_{CA}M+b_{CA}.
\]
On this simulator-induced coupling, for every \(\tau_A\),
\[
\{\mathbf J_A(T_A,\Phi_{CA}(e);M)\preceq\tau_A\}
\subseteq
\{\mathbf J_B(T_B,\Phi_{CB}(e);M_B)
  \preceq\Gamma_{BA}(\tau_A)\}
\subseteq
\{\mathbf J_C(\mathfrak C_{CB}(T_B),e;M_C)
  \preceq\Gamma_{CA}(\tau_A)\}
\]
up to null events. Thus all POR obligations hold for
\(C\preceq_{\mathrm{POR}}A\). If both reductions are exact, the composed
overhead is \((1,0)\).
}
\end{proof}

\subsection{Proof of Corollary~\ref{cor:equivalence-classes}}
\label{app:proof-equivalence-classes}

\begin{proof}
By definition, \(A\equiv_{\mathrm{POR}}B\) means both
\(A\succeq_{\mathrm{POR}}B\) and \(B\succeq_{\mathrm{POR}}A\). Reflexivity of
\(\equiv_{\mathrm{POR}}\) follows from reflexivity of
\(\succeq_{\mathrm{POR}}\). Symmetry is built into the definition. For
transitivity, if \(A\equiv_{\mathrm{POR}}B\) and
\(B\equiv_{\mathrm{POR}}C\), then transitivity of
\(\succeq_{\mathrm{POR}}\) gives \(A\succeq_{\mathrm{POR}}C\) and
\(C\succeq_{\mathrm{POR}}A\), so \(A\equiv_{\mathrm{POR}}C\).

It remains to check that the induced order on equivalence classes is well
defined. Suppose \(A\equiv_{\mathrm{POR}}A'\),
\(B\equiv_{\mathrm{POR}}B'\), and \(A\succeq_{\mathrm{POR}}B\). Then
\[
A'\succeq_{\mathrm{POR}}A\succeq_{\mathrm{POR}}B\succeq_{\mathrm{POR}}B',
\]
so \(A'\succeq_{\mathrm{POR}}B'\). Thus the comparison does not depend on the
chosen representatives. Reflexivity and transitivity descend from the
preorder. Antisymmetry holds because mutual comparability of two classes means
their representatives are POR-equivalent, and hence the two classes are the
same.
\end{proof}

\subsection{Proof of Proposition~\ref{prop:generic-templates}}
\label{app:proof-generic-templates}

\begin{proof}
We verify each template against the POR definition.

For (1), use the stated embedding \(\Phi\), compiler \(\mathfrak C\), threshold
map \(\Gamma\), and overhead \(a=1,b=0\). Fix
\(e,T,\tau_A,\delta,M\). If the \(A\)-success event has probability at least
\(1-\delta\), then the assumed event inclusion, up to null events, implies
\[
\Pr\!\left[
\mathbf J_B(\mathfrak C(T),e;M)\preceq\Gamma(\tau_A)
\right]
\ge
\Pr\!\left[
\mathbf J_A(T,\Phi(e);M)\preceq\tau_A
\right]
\ge 1-\delta .
\]
The assumed causal, protocol-valid execution coupling supplies compiler
admissibility, both protocol marginals, and the two declared-cost bounds.
This is exactly \(B\preceq_{\mathrm{ex}}A\), or equivalently
\(A\succeq_{\mathrm{ex}}B\).

For (2), the environments, protocols, learners, and resource accounting are
the same. Therefore the identity embedding and identity compiler are
admissible. With the given \(\Gamma\), the event-inclusion argument from (1)
applies verbatim under the identity execution coupling, so the reduction is exact.

For (3), the two paradigms differ only in that \(B\)'s protocol reveals extra
side information. Given any learner \(T\) for the less informative paradigm
\(A\), define \(\mathfrak C(T)\) to run the same decision rule under \(B\)'s
protocol while discarding the additional side information. After discarding
that information, the transcript seen by \(\mathfrak C(T)\) under \(B\) has
the same law as the transcript seen by \(T\) under \(A\). The objective,
environment class, resource accounting, and regime are the same, so the
success events agree in distribution under the identity threshold map. This
coupling preserves cost and gives an exact reduction.

For (4), every learner admissible for \(A\) is also admissible for \(B\). The
identity map on learners is therefore a valid compiler from \(\mathfrak T_A\)
to \(\mathfrak T_B\). Since all other components are the same, the success
events are identical under the identity embedding and identity threshold map.
The charged resource budget is unchanged, so the reduction is exact.
\end{proof}

\subsection{Proof of Proposition~\ref{prop:replay-monotonicity}}
\label{app:proof-replay-monotonicity}

\begin{proof}
The paradigms \(\mathrm{Cont}(r)\) and \(\mathrm{Cont}(r')\) have the same
environment class, observation protocol, performance functional, and sample
accounting rule. Their learner classes differ only in the replay-memory
constraint. If \(r\le r'\), any learner that stores at most \(r\) examples also
stores at most \(r'\) examples, so
\[
\mathfrak T_{\mathrm{Cont}(r)}
\subseteq
\mathfrak T_{\mathrm{Cont}(r')}.
\]
Proposition~\ref{prop:generic-templates}(4), applied with
\(A=\mathrm{Cont}(r)\) and \(B=\mathrm{Cont}(r')\), gives
\[
\mathrm{Cont}(r)\succeq_{\mathrm{ex}}\mathrm{Cont}(r').
\]
\end{proof}

\subsection{Proof of Proposition~\ref{prop:taskid-monotonicity}}
\label{app:proof-taskid-monotonicity}

\begin{proof}
The two continual-learning variants have the same environments, learner class,
performance functional, sample accounting rule, and accuracy regime. The only
difference is the observation protocol: \(\mathrm{Cont}^{\mathrm{id}}\) gives
the learner task identifiers, while \(\mathrm{Cont}^{\mathrm{blind}}\) does
not. A learner designed for the blind protocol can be run under the
identifier-revealing protocol by ignoring the identifiers. The resulting
transcript has the same law as in the blind protocol. Therefore
Proposition~\ref{prop:generic-templates}(3) applies and yields
\[
\mathrm{Cont}^{\mathrm{blind}}
\succeq_{\mathrm{ex}}
\mathrm{Cont}^{\mathrm{id}} .
\]
\end{proof}

\subsection{Proof of Corollary~\ref{cor:continual-variant-monotonicity}}
\label{app:proof-continual-variant-monotonicity}

\begin{proof}
For replay budgets, Proposition~\ref{prop:replay-monotonicity} gives the exact
reduction
\[
\mathrm{Cont}(r')\preceq_{\mathrm{ex}}\mathrm{Cont}(r)
\qquad\text{whenever }r\le r' .
\]
Applying Corollary~\ref{cor:exact-domination} with the identity embedding and
identity threshold map gives
\[
N_{\mathrm{Cont}(r')}^\star(\mathcal C;\tau,\delta)
\le
N_{\mathrm{Cont}(r)}^\star(\mathcal C;\tau,\delta).
\]

Similarly, Proposition~\ref{prop:taskid-monotonicity} gives
\[
\mathrm{Cont}^{\mathrm{id}}
\preceq_{\mathrm{ex}}
\mathrm{Cont}^{\mathrm{blind}} .
\]
The exact-domination corollary therefore implies
\[
N_{(\mathrm{Cont}^{\mathrm{id}})}^\star(\mathcal C;\tau,\delta)
\le
N_{(\mathrm{Cont}^{\mathrm{blind}})}^\star(\mathcal C;\tau,\delta).
\]
\end{proof}

\subsection{Proof of Proposition~\ref{prop:strict-objective}}
\label{app:proof-strict-objective}

\begin{proof}
First, \(P_{\mathrm{all}}\succeq_{\mathrm{ex}}P_{\mathrm{one}}\). The two
paradigms have the same environments, protocol, learners, and sample
accounting. Moreover,
\[
\{\mathbf J_{\mathrm{all}}(T,D_\sigma;M)\preceq(0,0)\}
\subseteq
\{\mathbf J_{\mathrm{one}}(T,D_\sigma;M)\preceq0\}.
\]
Thus Proposition~\ref{prop:generic-templates}(2) gives the exact reduction
\(P_{\mathrm{one}}\preceq_{\mathrm{ex}}P_{\mathrm{all}}\).

We now compute the two minimax quantities.

For \(P_{\mathrm{one}}\), consider the learner that memorizes the label of
input \(1\) if input \(1\) is observed, and otherwise guesses that label
uniformly at random. The event that input \(1\) is never observed in \(M\)
samples has probability \(2^{-M}\). Conditional on this event, the learner is
wrong with probability \(1/2\); otherwise it is correct. Hence its failure
probability is \(2^{-(M+1)}\), uniformly over \(\sigma\), and its success
probability is
\[
1-2^{-(M+1)}.
\]

No learner can have a better worst-case success probability. To see this,
place the uniform prior on \(\sigma_1\) while fixing \(\sigma_2\). On the
event that input \(1\) is not observed, the entire transcript is independent of
\(\sigma_1\). Therefore, conditioned on that event, any estimator of
\(\sigma_1\) has average success probability at most \(1/2\). Since the
worst-case success probability of a learner is no larger than its average
success probability under any prior, every learner has worst-case failure at
least \(2^{-(M+1)}\). The memorize-and-guess learner is therefore minimax
optimal for \(P_{\mathrm{one}}\).

This is a fixed-budget statement: at each \(M\), the displayed learner
succeeds with that probability for every \(D_\sigma\), while every learner has
some \(D_\sigma\) with at least the displayed failure probability. It therefore
matches the uniform quantifiers in~\eqref{eq:uniform-minimax}.

It follows that success probability at least \(1-\delta\) is possible exactly
when
\[
2^{-(M+1)}\le\delta,
\]
or equivalently
\[
M\ge \log_2\frac{1}{2\delta}.
\]
Since \(M\) is an integer,
\[
N_{P_{\mathrm{one}}}^\star(\mathcal C;0,\delta)
=
\left\lceil \log_2\frac{1}{2\delta}\right\rceil .
\]

For \(P_{\mathrm{all}}\), use the learner that memorizes all observed labels
and guesses each unobserved label uniformly. Since \(\delta<1/2\), a
zero-sample budget cannot be sufficient; indeed, with no samples the best
success probability for two unknown bits is \(1/4\). For \(M\ge 1\), either
both inputs have been observed or exactly one input has been observed. The
second case occurs precisely when all \(M\) samples fall on the same input,
which has probability \(2\cdot 2^{-M}=2^{1-M}\). Conditional on this case, the
unobserved label must be guessed, causing failure with probability \(1/2\).
Thus the memorize-and-guess learner has failure probability
\[
\frac12\,2^{1-M}=2^{-M},
\]
and success probability \(1-2^{-M}\).

The same prior argument proves optimality. Put the uniform prior on
\(\sigma=(\sigma_1,\sigma_2)\). Whenever one input is unobserved, its label is
still uniform and independent of the transcript, so any learner has
conditional average probability at most \(1/2\) of predicting both labels
correctly. Hence no learner can have worst-case failure below \(2^{-M}\) for
\(M\ge1\). Again, the lower bound chooses a hard environment for every
learner at the same fixed \(M\), exactly as uniform minimax complexity
requires. Therefore success probability at least \(1-\delta\) is possible
exactly when
\[
2^{-M}\le\delta,
\]
and
\[
N_{P_{\mathrm{all}}}^\star(\mathcal C;(0,0),\delta)
=
\left\lceil \log_2\frac{1}{\delta}\right\rceil .
\]

Finally, for \(\delta\in(0,1/2)\),
\[
\left\lceil \log_2\frac{1}{\delta}\right\rceil
=
\left\lceil \log_2\frac{1}{2\delta}+1\right\rceil
=
\left\lceil \log_2\frac{1}{2\delta}\right\rceil+1 .
\]
Thus the all-coordinate objective has strictly larger minimax sample
complexity.

It remains to rule out exact equivalence. Suppose, toward contradiction, that
there were an exact reverse reduction
\(P_{\mathrm{all}}\preceq_{\mathrm{ex}}P_{\mathrm{one}}\). The regimes are
singletons, so the threshold map must send \(0\) to \((0,0)\). For any
environment embedding \(\Phi:\mathcal C\to\mathcal C\), class monotonicity
gives
\[
N_{P_{\mathrm{one}}}^\star(\Phi(\mathcal C);0,\delta)
\le
N_{P_{\mathrm{one}}}^\star(\mathcal C;0,\delta)
<
N_{P_{\mathrm{all}}}^\star(\mathcal C;(0,0),\delta).
\]
This contradicts Corollary~\ref{cor:reduction-obstruction} with overhead
\((1,0)\). Hence no exact reverse reduction exists, and the two paradigms are
not exact-POR equivalent relative to the stated regimes.
\end{proof}

\subsection{Proof of Proposition~\ref{prop:scalable-objective-gap}}
\label{app:proof-scalable-objective-gap}

\begin{proof}
We first prove the upper bound for \(P_{\mathrm{one}}^{(m)}\). Use the learner
that memorizes the label of input \(1\) if it is observed and otherwise guesses
that label uniformly. The probability that input \(1\) is not observed in
\(M\) samples is \((1-1/m)^M\), so the failure probability is
\[
\frac12\left(1-\frac1m\right)^M .
\]
At \(M=2m\),
\[
\frac12\left(1-\frac1m\right)^{2m}
\le
\frac12 e^{-2}
<\frac18 .
\]
Thus
\[
N_{P_{\mathrm{one}}^{(m)}}^\star(\mathcal C_m;0,1/8)\le 2m .
\]

We next prove the lower bound for \(P_{\mathrm{all}}^{(m)}\). Let \(Z\) be the
number of inputs in \(\mathcal X_m\) not observed after \(M\) samples. Place
the uniform prior on \(\sigma\in\{0,1\}^m\). Conditional on any transcript with
\(Z>0\), at least one label remains completely unobserved and is still uniform
given the transcript. Therefore any learner has conditional average
probability at most \(1/2\) of predicting every label correctly. Consequently,
for every learner,
\[
\text{worst-case success}
\le
\text{average success under the prior}
\le
1-\frac12\Pr[Z>0].
\]
This worst-case statement is made at the same fixed budget \(M\): for
every learner, some \(D_\sigma\) has the displayed failure probability. This is
the quantifier order required by~\eqref{eq:uniform-minimax}.
If \(\Pr[Z>0]\ge1/2\), then the worst-case success probability is at most
\(3/4\), which is below the required \(7/8\). It remains to show that
\(\Pr[Z>0]\ge1/2\) whenever
\[
M\le \frac{m}{4}\log m .
\]

Let \(I_j\) be the indicator that input \(j\) is unobserved, so
\(Z=\sum_{j=1}^m I_j\). Then
\[
\mathbb E Z
=
m\left(1-\frac1m\right)^M .
\]
Since \(m\ge3\), \(1/m\le1/2\), and the elementary inequality
\(\log(1-u)\ge -2u\) holds for \(u\in(0,1/2]\). Therefore
\[
\mathbb E Z
\ge
m\exp\!\left(-\frac{2M}{m}\right)
\ge
m\exp\!\left(-\frac12\log m\right)
=
\sqrt m .
\]
For the second moment,
\[
\mathbb E Z^2
=
\sum_{j=1}^m\mathbb E I_j
+
\sum_{i\ne j}\mathbb E[I_iI_j].
\]
Here
\[
\mathbb E[I_iI_j]
=
\left(1-\frac2m\right)^M
\le
\left(1-\frac1m\right)^{2M},
\]
because \(1-2/m\le(1-1/m)^2\). Hence
\[
\mathbb E Z^2
\le
\mathbb E Z+(\mathbb E Z)^2 .
\]
By Cauchy's inequality,
\[
\Pr[Z>0]
\ge
\frac{(\mathbb E Z)^2}{\mathbb E Z^2}
\ge
\frac{\mathbb E Z}{1+\mathbb E Z}
\ge
\frac12,
\]
where the last step uses \(\mathbb E Z\ge\sqrt m\ge1\). Thus no learner can
achieve success probability \(7/8\) using \(M\le(m/4)\log m\) samples, proving
\[
N_{P_{\mathrm{all}}^{(m)}}^\star(\mathcal C_m;\mathbf 0_m,1/8)
\ge
\frac{m}{4}\log m .
\]

Finally, suppose for contradiction that
\(\bigl(P_{\mathrm{all}}^{(m)}\bigr)_{m\ge3}
\preceq_{\mathrm{uPOR}}
\bigl(P_{\mathrm{one}}^{(m)}\bigr)_{m\ge3}\). Then for every \(m\) there is a reverse reduction
\[
P_{\mathrm{all}}^{(m)}
\preceq_{\mathrm{POR}}
P_{\mathrm{one}}^{(m)}
\]
with the same affine overhead parameters \(a,b\), independent of \(m\),
while \(\Phi_m,\mathfrak C_m,\Gamma_m\) may vary with \(m\). Since
the regimes are singletons, the threshold map must send \(0\) to
\(\mathbf 0_m\). For any embedding
\(\Phi_m:\mathcal C_m\to\mathcal C_m\), class monotonicity and the upper bound
above imply
\[
N_{P_{\mathrm{one}}^{(m)}}^\star(\Phi_m(\mathcal C_m);0,1/8)
\le
N_{P_{\mathrm{one}}^{(m)}}^\star(\mathcal C_m;0,1/8)
\le 2m .
\]
Theorem~\ref{thm:reduction-domination} would therefore force
\[
N_{P_{\mathrm{all}}^{(m)}}^\star(\mathcal C_m;\mathbf 0_m,1/8)
\le 2am+b .
\]
But the lower bound just proved gives
\[
N_{P_{\mathrm{all}}^{(m)}}^\star(\mathcal C_m;\mathbf 0_m,1/8)
\ge \frac{m}{4}\log m ,
\]
and \((m/4)\log m>2am+b\) for all sufficiently large \(m\). This
contradiction rules out the stated uniform-overhead family reduction. It
does not rule out \(m\)-dependent coefficients for fixed \(m\).
\end{proof}

\subsection{Proof of Proposition~\ref{prop:transfer-dominates-supervised}}
\label{app:proof-transfer-dominates-supervised}

\begin{proof}
{\color{black}
We verify all POR obligations for
\(\mathrm{Sup}\preceq_{\mathrm{ex}}\mathrm{Trans}_K\). Embed
\[
\Phi_{\mathrm{Sup}\to\mathrm{Trans}_K}(D)
=(D,\ldots,D;D).
\]
Given any \(T_{\mathrm{Trans}_K}\in\mathfrak T_{\mathrm{Trans}_K}\) and
external budget \(M\), first obtain its preallocation declaration
\(\bar M:=q_{T_{\mathrm{Trans}_K}}(M)\). Define
\(\mathfrak C(T_{\mathrm{Trans}_K})\) to relay \(\bar M\) to the supervised
protocol before requesting data, partition its \(\bar M\)-example transcript
according to \(\mathbf m_K(\bar M)\), add the prescribed phase tags, run
\(T_{\mathrm{Trans}_K}\), and return its output predictor. The tagging map is
causal and independent of \(D\); canonical closure in
Assumption~\ref{ass:degenerate-transfer} therefore gives
\[
\mathfrak C(T_{\mathrm{Trans}_K})\in\mathfrak T_{\mathrm{Sup}}.
\]
The same assumption gives the correct external-budget-\(M\) transfer-protocol
marginal, including the solver's cap, under the coupling. Both executions
charge exactly the same \(\bar M\) revealed examples, so their costs are at most
\(\bar M\le M\). Finally, their output predictors coincide and hence
\[
\mathbf J_{\mathrm{Sup}}(
\mathfrak C(T_{\mathrm{Trans}_K}),D;M)
=
\mathbf J_{\mathrm{Trans}_K}(T_{\mathrm{Trans}_K},
\Phi_{\mathrm{Sup}\to\mathrm{Trans}_K}(D);M)
\]
under the coupling. The identity threshold map and overhead \((1,0)\) complete
the exact reduction, equivalently
\(\mathrm{Trans}_K\succeq_{\mathrm{ex}}\mathrm{Sup}\).
}
\end{proof}

\subsection{Proof of Proposition~\ref{prop:continual-dominates-transfer}}
\label{app:proof-continual-dominates-transfer}

\begin{proof}
{\color{black}
We verify
\(\mathrm{Trans}_K\preceq_{\mathrm{ex}}\mathrm{Cont}_{K+1}\).
Let
\[
e=(D_1,\dots,D_K;D_\star)
\]
be a transfer-learning environment. By
Assumption~\ref{ass:continual-target}, the embedding
\[
\Phi_{\mathrm{Trans}_K\to\mathrm{Cont}_{K+1}}(e)
:=
(D_1,\dots,D_K,D_\star)
\]
is a valid continual-learning environment.

Given any \(T_{\mathrm{Cont}_{K+1}}\) and external budget \(M\), first obtain
\(\bar M:=q_{T_{\mathrm{Cont}_{K+1}}}(M)\). Define the compiled transfer
learner to relay \(\bar M\) before phase allocation, then causally relabel
source phase \(j\) as task \(j\), relabel the target phase as task \(K+1\), run
\(T_{\mathrm{Cont}_{K+1}}\), and return its final predictor.
Assumption~\ref{ass:continual-target} gives
\[
\mathfrak C(T_{\mathrm{Cont}_{K+1}})\in\mathfrak T_{\mathrm{Trans}_K},
\]
and states that the relabeled transcript, allocated according to
\(\mathbf m_K(\bar M)\), has the correct external-budget-\(M\) continual
marginal. The sample blocks and charged examples are identical, so both costs
are at most \(\bar M\le M\), and the output predictors coincide under the coupling.
Let \(\Gamma\) project a continual-learning threshold vector to the coordinate
that bounds the final-task risk. Use overhead \(a=1,b=0\).

If
\[
\mathbf J_{\mathrm{Cont}_{K+1}}(T_{\mathrm{Cont}_{K+1}},
\Phi_{\mathrm{Trans}_K\to\mathrm{Cont}_{K+1}}(e);M)\preceq\tau_{\mathrm{Cont}},
\]
then, in particular, the final-task risk coordinate is at most
\(\Gamma(\tau_{\mathrm{Cont}})\). That coordinate is exactly the transfer
objective of the compiled learner. Hence
\[
\{\mathbf J_{\mathrm{Cont}_{K+1}}(T_{\mathrm{Cont}_{K+1}},
\Phi_{\mathrm{Trans}_K\to\mathrm{Cont}_{K+1}}(e);M)\preceq\tau_{\mathrm{Cont}}\}
\subseteq
\{\mathbf J_{\mathrm{Trans}_K}(\mathfrak C(T_{\mathrm{Cont}_{K+1}}),e;M)
\preceq\Gamma(\tau_{\mathrm{Cont}})\}.
\]
The event inclusion holds on the explicit relabeling coupling and, together
with admissibility and cost equality, verifies the exact POR-reduction.
}
\end{proof}

\subsection{Proof of Corollary~\ref{cor:continual-dominates-supervised}}
\label{app:proof-continual-dominates-supervised}

\begin{proof}
Proposition~\ref{prop:transfer-dominates-supervised} gives
\[
\mathrm{Trans}_K\succeq_{\mathrm{ex}}\mathrm{Sup},
\]
and Proposition~\ref{prop:continual-dominates-transfer} gives
\[
\mathrm{Cont}_{K+1}\succeq_{\mathrm{ex}}\mathrm{Trans}_K.
\]
By the transitivity part of Proposition~\ref{prop:preorder},
\[
\mathrm{Cont}_{K+1}\succeq_{\mathrm{ex}}\mathrm{Sup}.
\]
\end{proof}

\subsection{Proof of Proposition~\ref{prop:meta-dominates-supervised}}
\label{app:proof-meta-dominates-supervised}

\begin{proof}
{\color{black}
The desired statement is equivalent to
\[
\mathrm{Sup}\preceq_{\mathrm{ex}}\mathrm{Meta}.
\]
By Assumption~\ref{ass:singleton-meta}, every supervised environment \(D\) has
a singleton-task meta-learning embedding
\[
\Phi_{\mathrm{Sup}\to\mathrm{Meta}}(D):=\Pi_D=\delta_D .
\]
Given any meta learner \(T_{\mathrm{Meta}}\) and external budget \(M\), first
obtain \(\bar M:=q_{T_{\mathrm{Meta}}}(M)\). Define the compiled supervised
learner \(\mathfrak C(T_{\mathrm{Meta}})\) to relay \(\bar M\) before requesting
data, buffer only \(\bar M\) examples from its i.i.d. stream, apply the
deterministic causal reblocking map \(\rho_{\bar M}\), feed that transcript to
\(T_{\mathrm{Meta}}\), and return the predictor obtained after adaptation on
the reblocked evaluation support \(S_{\mathrm{ev}}\). By
Assumption~\ref{ass:singleton-meta}, \(\rho_{\bar M}\) selects and reblocks at most
\(\bar M\) charged examples into a meta-training block and evaluation support satisfying
\[
m_{\mathrm{train}}(\bar M)+|S_{\mathrm{ev}}|\le \bar M.
\]
Moreover,
\[
\mathfrak C(T_{\mathrm{Meta}})\in\mathfrak T_{\mathrm{Sup}},
\]
the reblocked learner-visible transcript---including the meta-training block,
anonymous evaluation-phase marker, and evaluation support---has the correct
meta-protocol marginal. The full coupled execution additionally contains the
environment-side latent evaluation task supplied by \(\delta_D\) and auxiliary
evaluation randomness; neither the reblocking map nor the compiler emits the
task identity. Its charged cost is at most \(\bar M\le M\). Use the identity threshold
map and overhead \((1,0)\).

Under the singleton task distribution \(\Pi_D\), every episode is drawn from
the same underlying task \(D\), and the latent evaluation task satisfies
\(D_{\mathrm{ev}}=D\) almost surely. The reblocking coupling gives the
almost-sure predictor identity
\[
h_{\mathfrak C(T_{\mathrm{Meta}}),D}^{M}
=
A_{T_{\mathrm{Meta}},\delta_D}^{m_{\mathrm{train}}(\bar M)}
(S_{\mathrm{ev}}).
\]
Consequently,
\begin{align*}
\mathbf J_{\mathrm{Sup}}(\mathfrak C(T_{\mathrm{Meta}}),D;M)
&=
R_D\!\left(h_{\mathfrak C(T_{\mathrm{Meta}}),D}^{M}\right)\\
&=
R_{D_{\mathrm{ev}}}\!\left(
A_{T_{\mathrm{Meta}},\delta_D}^{m_{\mathrm{train}}(\bar M)}
(S_{\mathrm{ev}})\right)\\
&=
\mathbf J_{\mathrm{Meta}}(T_{\mathrm{Meta}},\delta_D;M)
\end{align*}
almost surely. Thus the random performance variables agree under the charged
reblocking coupling at the same raw-example budget. This verifies the exact POR-reduction
\(\mathrm{Sup}\preceq_{\mathrm{ex}}\mathrm{Meta}\), equivalently
\(\mathrm{Meta}\succeq_{\mathrm{ex}}\mathrm{Sup}\).
}
\end{proof}

\subsection{Proof of Corollary~\ref{cor:canonical-sample-hierarchy}}
\label{app:proof-canonical-sample-hierarchy}

\begin{proof}
Each inequality is an instance of exact sample-complexity domination.

First, Proposition~\ref{prop:transfer-dominates-supervised} gives
\(\mathrm{Sup}\preceq_{\mathrm{ex}}\mathrm{Trans}_K\) with embedding
\(\Phi_{\mathrm{Sup}\to\mathrm{Trans}_K}\) and the identity threshold map. Applying
Corollary~\ref{cor:exact-domination} gives
\[
N^\star_{\mathrm{Sup}}(\mathcal C_{\mathrm{Sup}};
\tau_{\mathrm{Sup}},\delta)
\le
N^\star_{\mathrm{Trans}_K}(
\Phi_{\mathrm{Sup}\to\mathrm{Trans}_K}(\mathcal C_{\mathrm{Sup}});
\tau_{\mathrm{Sup}},\delta).
\]

Second, Corollary~\ref{cor:continual-dominates-supervised} gives the composed
exact reduction from supervised learning to continual learning. The composed
environment embedding is
\[
\Phi_{\mathrm{Trans}_K\to\mathrm{Cont}_{K+1}}\circ
\Phi_{\mathrm{Sup}\to\mathrm{Trans}_K},
\]
and the relevant threshold map is the final-task projection. Thus, for every
continual threshold \(\tau_{\mathrm{Cont}}\) whose final-task projection equals
\(\tau_{\mathrm{Sup}}\), exact domination gives the displayed continual
inequality.

Third, Proposition~\ref{prop:meta-dominates-supervised} gives
\(\mathrm{Sup}\preceq_{\mathrm{ex}}\mathrm{Meta}\) with embedding
\(\Phi_{\mathrm{Sup}\to\mathrm{Meta}}\) and the identity scalar-risk threshold
map. Applying Corollary~\ref{cor:exact-domination} gives the displayed
meta-learning inequality.
\end{proof}

%% file: reference.bib
@inproceedings{su2023towards,
  title={Towards robust graph incremental learning on evolving graphs},
  author={Su, Junwei and Zou, Difan and Zhang, Zijun and Wu, Chuan},
  booktitle={International Conference on Machine Learning},
  pages={32728--32748},
  year={2023},
  organization={PMLR}
}

@inproceedings{su2024pres,
  title={Pres: Toward scalable memory-based dynamic graph neural networks},
  author={Su, Junwei and Zou, Difan and Wu, Chuan},
  booktitle={International Conference on Learning Representations},
  volume={2024},
  pages={55675--55703},
  year={2024}
}

@article{su2023limitation,
  title={On the limitation and experience replay for GNNs in continual learning},
  author={Su, Junwei and Zou, Difan and Wu, Chuan},
  journal={arXiv preprint arXiv:2302.03534},
  year={2023}
}

@inproceedings{su2024topology,
  title={On the topology awareness and generalization performance of graph neural networks},
  author={Su, Junwei and Wu, Chuan},
  booktitle={European Conference on Computer Vision},
  pages={73--89},
  year={2024},
  organization={Springer}
}

@article{su2025interplay,
  title={On the interplay between graph structure and learning algorithms in graph neural networks},
  author={Su, Junwei and Wu, Chuan},
  journal={arXiv preprint arXiv:2508.14338},
  year={2025}
}

@article{su2026improving,
  title={Improving Implicit Regularization of SGD with Preconditioning for Least Square Problems},
  author={Su, Junwei and Wu, Chuan and Zheng, Le and Zou, Difan},
  journal={SIAM Journal on Mathematics of Data Science},
  volume={8},
  number={3},
  pages={599--622},
  year={2026},
  publisher={SIAM}
}

@article{liu2026full,
  title={Full-Graph vs. Mini-Batch Training: Comprehensive Analysis from a Batch Size and Fan-Out Size Perspective},
  author={Liu, Mengfan and Zheng, Da and Su, Junwei and Wu, Chuan},
  journal={arXiv preprint arXiv:2601.22678},
  year={2026}
}

@article{su2026multi,
  title={When Do Multi-Agent Systems Outperform? Analysing the Learning Efficiency of Agentic Systems},
  author={Su, Junwei and Wu, Chuan},
  journal={arXiv preprint arXiv:2602.08272},
  year={2026}
}

@article{su2025non,
  title={A non-asymptotic convergent analysis for scored-based graph generative model via a system of stochastic differential equations},
  author={Su, Junwei and Wu, Chuan},
  journal={arXiv preprint arXiv:2508.14351},
  year={2025}
}

@inproceedings{sheng2024mspipe,
  title={Mspipe: Efficient temporal gnn training via staleness-aware pipeline},
  author={Sheng, Guangming and Su, Junwei and Huang, Chao and Wu, Chuan},
  booktitle={Proceedings of the 30th ACM SIGKDD Conference on Knowledge Discovery and Data Mining},
  pages={2651--2662},
  year={2024}
}

@inproceedings{su2025temporal,
  title={Temporal-aware evaluation and learning for temporal graph neural networks},
  author={Su, Junwei and Wu, Shan},
  booktitle={Proceedings of the AAAI Conference on Artificial Intelligence},
  volume={39},
  number={19},
  pages={20628--20636},
  year={2025}
}

@article{su2024bg,
  title={BG-HGNN: Toward Efficient Learning for Complex Heterogeneous Graphs},
  author={Su, Junwei and Mao, Lingjun and Da, Zheng and Wu, Chuan},
  journal={arXiv preprint arXiv:2403.08207},
  year={2024}
}

@article{shen2026structuring,
  title={Structuring Semantic Embeddings for Principle Evaluation: A Prototype-Guided Contrastive Learning Approach},
  author={Shen, Che and Su, Junwei and Kong, Lingpeng and Wu, Chuan},
  journal={arXiv preprint arXiv:2608.15224},
  year={2026}
}

@article{vapnik1998statistical,
  title={Statistical learning theory},
  author={Vapnik, Vladimir},
  journal={John Wiley \& Sons},
  volume={2},
  pages={831--842},
  year={1998}
}

@article{valiant1984theory,
  title={A theory of the learnable},
  author={Valiant, LG},
  journal={Communications of the ACM},
  volume={27},
  number={11},
  pages={1134--1142},
  year={1984},
  publisher={Association for Computing Machinery (ACM)}
}

@article{vapnik1971uniform,
  title={On the Uniform Convergence of Relative Frequencies of Events to Their Probabilities},
  author={Vapnik, VN and Chervonenkis, A Ya},
  journal={Theory of Probability \& Its Applications},
  volume={16},
  number={2},
  pages={264--280},
  year={1971},
  publisher={Society for Industrial \& Applied Mathematics (SIAM)}
}

@article{blumer1989learnability,
  title={Learnability and the Vapnik-Chervonenkis dimension},
  author={Blumer, Anselm and Haussler, David and Warmuth, Manfred K},
  journal={Journal of the ACM (JACM)},
  volume={36},
  number={4},
  pages={929--965},
  year={1989},
  publisher={ACM New York, NY, USA}
}

@book{shalev2014understanding,
  title={Understanding Machine Learning: From Theory to Algorithms},
  author={Shalev-Shwartz, Shai and Ben-David, Shai},
  year={2014},
  publisher={Cambridge University Press}
}

@book{mohri2018foundations,
  title={Foundations of Machine Learning},
  author={Mohri, Mehryar and Rostamizadeh, Afshin and Talwalkar, Ameet},
  year={2018},
  publisher={MIT Press}
}

@book{anthony1999neural,
  title={Neural network learning: theoretical foundations},
  author={Anthony, Martin and Bartlett, P},
  year={1999},
  publisher={Cambridge University Press}
}

@article{bartlett2002rademacher,
  title={Rademacher and gaussian complexities: Risk bounds and structural results},
  author={Bartlett, Peter L and Mendelson, Shahar},
  journal={Journal of machine learning research},
  volume={3},
  number={Nov},
  pages={463--482},
  year={2002}
}

@article{bousquet2002stability,
  title={Stability and generalization},
  author={Bousquet, Olivier and Elisseeff, Andr{\'e}},
  journal={Journal of machine learning research},
  volume={2},
  number={Mar},
  pages={499--526},
  year={2002}
}

@article{bengio2013representation,
  title={Representation learning: A review and new perspectives},
  author={Bengio, Yoshua and Courville, Aaron and Vincent, Pascal},
  journal={IEEE transactions on pattern analysis and machine intelligence},
  volume={35},
  number={8},
  pages={1798--1828},
  year={2013},
  publisher={IEEE}
}

@article{caruana1997multitask,
  title={Multitask learning},
  author={Caruana, Rich},
  journal={Machine learning},
  volume={28},
  number={1},
  pages={41--75},
  year={1997},
  publisher={Springer}
}

@article{pan2010survey,
  title={A Survey on Transfer Learning},
  author={Pan, Sinno Jialin and Yang, Qiang},
  journal={IEEE Transactions on Knowledge and Data Engineering},
  volume={22},
  number={10},
  pages={1345--1359},
  year={2010},
  publisher={IEEE Educational Activities Department Piscataway, NJ, USA}
}

@article{bendavid2010theory,
  title={A theory of learning from different domains},
  author={Ben-David, Shai and Blitzer, John and Crammer, Koby and Kulesza, Alex and Pereira, Fernando and Vaughan, Jennifer Wortman},
  journal={Machine learning},
  volume={79},
  number={1},
  pages={151--175},
  year={2010},
  publisher={Springer}
}

@article{crammer2008learning,
  title={Learning from Multiple Sources.},
  author={Crammer, Koby and Kearns, Michael and Wortman, Jennifer},
  journal={Journal of machine learning research},
  volume={9},
  number={8},
  year={2008}
}

@article{blitzer2008learning,
  title={Learning bounds for domain adaptation},
  author={Blitzer, John and Crammer, Koby and Kulesza, Alex and Pereira, Fernando and Wortman, Jennifer},
  journal={Advances in neural information processing systems},
  volume={20},
  year={2007}
}

@article{mansour2009domain,
  title={Domain adaptation with multiple sources},
  author={Mansour, Yishay and Mohri, Mehryar and Rostamizadeh, Afshin},
  journal={Advances in neural information processing systems},
  volume={21},
  year={2008}
}

@article{weiss2016survey,
  title={A survey of transfer learning},
  author={Weiss, Karl and Khoshgoftaar, Taghi M and Wang, DingDing},
  journal={Journal of Big data},
  volume={3},
  number={1},
  pages={9},
  year={2016},
  publisher={Springer}
}

@article{parisi2019continual,
  title={Continual lifelong learning with neural networks: A review},
  author={Parisi, German I and Kemker, Ronald and Part, Jose L and Kanan, Christopher and Wermter, Stefan},
  journal={Neural networks},
  volume={113},
  pages={54--71},
  year={2019},
  publisher={Elsevier}
}

@article{robins1995catastrophic,
  title={Catastrophic forgetting, rehearsal and pseudorehearsal},
  author={Robins, Anthony},
  journal={Connection Science},
  volume={7},
  number={2},
  pages={123--146},
  year={1995},
  publisher={Taylor \& Francis}
}

@article{french1999catastrophic,
  title={Catastrophic forgetting in connectionist networks},
  author={French, Robert M},
  journal={Trends in cognitive sciences},
  volume={3},
  number={4},
  pages={128--135},
  year={1999},
  publisher={Elsevier}
}

@inproceedings{rebuffi2017icarl,
  title={icarl: Incremental classifier and representation learning},
  author={Rebuffi, Sylvestre-Alvise and Kolesnikov, Alexander and Sperl, Georg and Lampert, Christoph H},
  booktitle={Proceedings of the IEEE conference on Computer Vision and Pattern Recognition},
  pages={2001--2010},
  year={2017}
}

@article{delange2022continual,
  title={A Continual Learning Survey: Defying Forgetting in Classification Tasks},
  author={De Lange, Matthias and Aljundi, Rahaf and Masana, Marc and Parisot, Sarah and Jia, Xu and Leonardis, Ales and Slabaugh, Gregory and Tuytelaars, Tinne},
  journal={IEEE Transactions on Pattern Analysis \& Machine Intelligence},
  volume={44},
  number={07},
  pages={3366--3385},
  year={2022},
  publisher={IEEE Computer Society}
}

@article{vanderven2022three,
  title={Three types of incremental learning},
  author={Van de Ven, Gido M and Tuytelaars, Tinne and Tolias, Andreas S},
  journal={Nature Machine Intelligence},
  volume={4},
  number={12},
  pages={1185--1197},
  year={2022},
  publisher={Nature Publishing Group UK London}
}

@article{pentina2015lifelong,
  title={Lifelong learning with non-iid tasks},
  author={Pentina, Anastasia and Lampert, Christoph H},
  journal={Advances in Neural Information Processing Systems},
  volume={28},
  year={2015}
}

@article{hospedales2022meta,
  title={Meta-learning in neural networks: A survey},
  author={Hospedales, Timothy and Antoniou, Antreas and Micaelli, Paul and Storkey, Amos},
  journal={IEEE transactions on pattern analysis and machine intelligence},
  volume={44},
  number={9},
  pages={5149--5169},
  year={2021},
  publisher={IEEE}
}

@inproceedings{cook1971complexity,
  title={The complexity of theorem-proving procedures},
  author={Cook, Stephen A},
  booktitle={Proceedings of the third annual ACM symposium on Theory of computing},
  pages={151--158},
  year={1971}
}

@article{karp1972reducibility,
  title={Reducibility among combinatorial problems},
  author={Karp, Richard M.},
  journal={Complexity of Computer Computations},
  pages={85--103},
  year={1972}
}

@misc{garey1979computers,
  title={Computers and Intractability: A Guide to the Theory of NP-Completeness},
  author={Garey, Michael R and Johnson, David S},
  year={1979},
  publisher={WH Freeman \& Co.}
}

@book{papadimitriou1994computational,
  title={Computational complexity},
  author={Papadimitriou, Christos H},
  year={1994},
  publisher={Addison-Wesley},
  address={Reading, MA}
}

@book{arora2009computational,
  title={Computational complexity: a modern approach},
  author={Arora, Sanjeev and Barak, Boaz},
  year={2009},
  publisher={Cambridge University Press}
}

@article{kearns1998statistical,
  title={Efficient noise-tolerant learning from statistical queries},
  author={Kearns, Michael},
  journal={Journal of the ACM (JACM)},
  volume={45},
  number={6},
  pages={983--1006},
  year={1998},
  publisher={ACM New York, NY, USA}
}

@inproceedings{beygelzimer2005error,
  title={Error limiting reductions between classification tasks},
  author={Beygelzimer, Alina and Dani, Varsha and Hayes, Tom and Langford, John and Zadrozny, Bianca},
  booktitle={Proceedings of the 22nd international conference on Machine learning},
  pages={49--56},
  year={2005}
}

@article{baxter2000model,
  title={A Model of Inductive Bias Learning},
  author={Baxter, J},
  journal={The Journal of Artificial Intelligence Research},
  volume={12},
  pages={149},
  year={2000},
  publisher={AI Access Foundation}
}

@incollection{thrun1998lifelong,
  title={Lifelong learning algorithms},
  author={Thrun, Sebastian},
  booktitle={Learning to learn},
  pages={181--209},
  publisher={Springer},
  year={1998}
}

@article{maurer2005stability,
  title={Algorithmic stability and meta-learning.},
  author={Maurer, Andreas and Jaakkola, Tommi},
  journal={Journal of Machine Learning Research},
  volume={6},
  number={6},
  year={2005}
}

@article{maurer2016benefit,
  title={The benefit of multitask representation learning},
  author={Maurer, Andreas and Pontil, Massimiliano and Romera-Paredes, Bernardino},
  journal={Journal of Machine Learning Research},
  volume={17},
  number={81},
  pages={1--32},
  year={2016}
}

@inproceedings{finn2017model,
  title={Model-agnostic meta-learning for fast adaptation of deep networks},
  author={Finn, Chelsea and Abbeel, Pieter and Levine, Sergey},
  booktitle={International conference on machine learning},
  pages={1126--1135},
  year={2017},
  organization={PMLR}
}

@article{lake2015human,
  title={Human-level concept learning through probabilistic program induction},
  author={Lake, Brenden M and Salakhutdinov, Ruslan and Tenenbaum, Joshua B},
  journal={Science},
  volume={350},
  number={6266},
  pages={1332--1338},
  year={2015},
  publisher={American Association for the Advancement of Science}
}

@article{andrychowicz2016learning,
  title={Learning to learn by gradient descent by gradient descent},
  author={Andrychowicz, Marcin and Denil, Misha and Gomez, Sergio and Hoffman, Matthew W and Pfau, David and Schaul, Tom and Shillingford, Brendan and De Freitas, Nando},
  journal={Advances in neural information processing systems},
  volume={29},
  year={2016}
}

@inproceedings{santoro2016meta,
  title={Meta-learning with memory-augmented neural networks},
  author={Santoro, Adam and Bartunov, Sergey and Botvinick, Matthew and Wierstra, Daan and Lillicrap, Timothy},
  booktitle={International conference on machine learning},
  pages={1842--1850},
  year={2016},
  organization={PMLR}
}

@article{vinyals2016matching,
  title={Matching networks for one shot learning},
  author={Vinyals, Oriol and Blundell, Charles and Lillicrap, Timothy and Wierstra, Daan and others},
  journal={Advances in neural information processing systems},
  volume={29},
  year={2016}
}

@inproceedings{ravi2017optimization,
  title={Optimization as a model for few-shot learning},
  author={Ravi, Sachin and Larochelle, Hugo},
  booktitle={International conference on learning representations},
  year={2017}
}

@inproceedings{snell2017prototypical,
  title={Prototypical networks for few-shot learning},
  author={Snell, Jake and Swersky, Kevin and Zemel, Richard},
  booktitle={Proceedings of the 31st International Conference on Neural Information Processing Systems},
  pages={4080--4090},
  year={2017}
}

@inproceedings{tripuraneni2020task,
  title={On the theory of transfer learning: the importance of task diversity},
  author={Tripuraneni, Nilesh and Jordan, Michael I and Jin, Chi},
  booktitle={Proceedings of the 34th International Conference on Neural Information Processing Systems},
  pages={7852--7862},
  year={2020}
}

@inproceedings{kalan2020minimax,
  title={Minimax lower bounds for transfer learning with linear and one-hidden layer neural networks},
  author={Kalan, Seyed Mohammadreza Mousavi and Fabian, Zalan and Avestimehr, Salman and Soltanolkotabi, Mahdi},
  booktitle={Proceedings of the 34th International Conference on Neural Information Processing Systems},
  pages={1959--1969},
  year={2020}
}

@article{kirkpatrick2017overcoming,
  title={Overcoming catastrophic forgetting in neural networks},
  author={Kirkpatrick, James and Pascanu, Razvan and Rabinowitz, Neil and Veness, Joel and Desjardins, Guillaume and Rusu, Andrei A and Milan, Kieran and Quan, John and Ramalho, Tiago and Grabska-Barwinska, Agnieszka and others},
  journal={Proceedings of the national academy of sciences},
  volume={114},
  number={13},
  pages={3521--3526},
  year={2017},
  publisher={National Academy of Sciences}
}

@inproceedings{lopezpaz2017gradient,
  title={Gradient episodic memory for continual learning},
  author={Lopez-Paz, David and Ranzato, Marc'Aurelio},
  booktitle={Proceedings of the 31st International Conference on Neural Information Processing Systems},
  pages={6470--6479},
  year={2017}
}

@article{blackwell1953equivalent,
  title={Equivalent comparisons of experiments},
  author={Blackwell, David},
  journal={The annals of mathematical statistics},
  pages={265--272},
  year={1953},
  publisher={JSTOR}
}

@article{lecam1964sufficiency,
  title={Sufficiency and approximate sufficiency},
  author={Le, L},
  journal={The Annals of Mathematical Statistics},
  pages={1419--1455},
  year={1964},
  publisher={JSTOR}
}

@book{lecam1986asymptotic,
  title={Asymptotic methods in statistical decision theory},
  author={Le Cam, Lucien},
  year={2012},
  publisher={Springer Science \& Business Media}
}

@book{torgersen1991comparison,
  title={Comparison of statistical experiments},
  author={Torgersen, Erik},
  volume={36},
  year={1991},
  publisher={Cambridge University Press}
}
